\documentclass[11pt, twoside, letter]{article}
\usepackage[margin = 1in]{geometry}
\usepackage{array}
\usepackage{amsthm}
\usepackage{amsmath}
\usepackage{amsfonts}
\usepackage{nicefrac} 
\usepackage{microtype}
\usepackage{graphicx}
\usepackage{comment}
\usepackage{subfigure}
\usepackage{booktabs} 
\newtheorem{theorem}{Theorem}[section]
\newtheorem{lemma}[theorem]{Lemma}
\newtheorem{proposition}[theorem]{Proposition}

\newtheorem{assumption}{Assumption}
\theoremstyle{definition}
\newtheorem{definition}[theorem]{Definition}

\usepackage{xcolor}
\usepackage{pgfplots}
\usepackage{xr}
\usepackage{mathtools}

\newtheorem{innercustomgeneric}{\customgenericname}
\providecommand{\customgenericname}{}
\newcommand{\newcustomtheorem}[2]{%
  \newenvironment{#1}[1]
  {%
   \renewcommand\customgenericname{#2}%
   \renewcommand\theinnercustomgeneric{##1}%
   \innercustomgeneric
  }
  {\endinnercustomgeneric}
}

\newcustomtheorem{customthm}{Theorem}
\newcustomtheorem{customlemma}{Lemma}

\theoremstyle{remark}

\newcommand{\abs}[1]{\ensuremath{\left|#1\right|}}
\newcommand{\norm}[2][]{\ensuremath{\left\Vert #2 \right\Vert}}

\renewcommand{\vec}[1]{\mathbf{#1}}

\newcommand{\vol}{\mathrm{Vol}}
\newcommand{\grad}{\mathrm{grad}}
\newcommand{\Proj}{\mathrm{Proj}}
\newcommand{\Retr}{\mathrm{Retr}}

\newcommand{\Hess}{\mathrm{Hess}}
\newcommand{\Exp}{\mathrm{Exp}}
\newcommand{\tr}{\mathrm{Tr}}
\newcommand{\Div}{\mathrm{div}}
\newcommand\pder[2][]{\ensuremath{\frac{\partial#1}{\partial#2}}}

\begin{document}

\title{Fast Convergence of Langevin Dynamics on Manifold: Geodesics meet Log-Sobolev}

\author{Xiao Wang\\SUTD\\xiao\_wang@sutd.edu.sg
\and Qi Lei\\Princeton\\qilei@princeton.edu
\and Ioannis Panageas\\UC Irvine\\ipanagea@uci.edu
}


\date{}

\maketitle

\begin{abstract}
Sampling is a fundamental and arguably very important task with numerous applications in Machine Learning. One approach to sample from a high dimensional distribution $e^{-f}$ for some function $f$ is the Langevin Algorithm (LA). Recently, there has been a lot of progress in showing fast convergence of LA even in cases where $f$ is non-convex, notably \cite{VW19}, \cite{MoritaRisteski} in which the former paper focuses on functions $f$ defined in $\mathbb{R}^n$ and the latter paper focuses on functions with symmetries (like matrix completion type objectives) with manifold structure. Our work generalizes the results of \cite{VW19} where $f$ is defined on a manifold $M$ rather than $\mathbb{R}^n$. From technical point of view, we show that KL decreases in a geometric rate whenever the distribution $e^{-f}$ satisfies a log-Sobolev inequality on $M$.  
\end{abstract}

\section{Introduction}
We focus on the problem of sampling from a distribution $e^{-f(x)}$ supported on a Riemannian manifold $M$ with standard volume measure. Sampling is a fundamental and arguably very important task with numerous applications in machine learning and Langevin dynamics is a quite standard approach. There is a growing interest in Langevin algorithms, e.g. \cite{Welling, Wibisono18, ArnakD}, due to its simple structure and the good empirical behavior. The classic Riemannian Langevin algorithm, e.g. \cite{GC11, PattersonTeh, ZPFP}, is used to sample from distributions supported on $\mathbb{R}^n$ (or a subset $D$) by endowing $\mathbb{R}^n$ (or $D$) a Riemannian structure. Beyond the classic application of Riemannian Langevin Algorithm (RLA), recent progress in \cite{DJMRB, MoritaRisteski, LiErdogdu} shows that sampling from a distribution on a manifold has application in matrix factorization, principal component analysis, matrix completion, solving SDP, mean field and continuous games and GANs. Formally, a game with finite number of agents is called continuous if the strategy spaces are continuous, either a finite dimensional differential manifold or an infinite dimensional Banach manifold \cite{RBS13, RBS16, DJMRB}. The mixed strategy is then a probability distribution on the strategy manifold and mixed Nash equilibria can be approximated by Langevin dynamics. 

\paragraph{Geodesic Langevin Algorithm (GLA).} In order to sample from a distribution on $M$, geodesic based algorithms (e.g. Geodesic Monte Carlo and Geodesic MCMC) are considered in \cite{BG13, LZS16}, where a geodesic integrator is used in the implementation.
We propose a Geodesic Langevin Algorithm (GLA) as a natural generalization of unadjusted Langevin algorithm (ULA) from the Euclidean space to manifold $M$. The benefit of GLA is to leverage sufficiently the geometric information (curvature, geodesic distance, isoperimetry) of $M$ while keeping the structure of the algorithm simple enough, so that we can obtain a non-asymptotic convergence guarantee of the algorithm.
In local coordinate systems, the Riemannian metric is represented by a matrix $g=\{g_{ij}\}$, see Definition \ref{rmetric}. We denote $g^{ij}$ the $ij$-th entry of the inverse matrix $g^{-1}$ of $g$, and $\abs{g}=\det(g_{ij})$, the determinant of the matrix $\{g_{ij}\}$. Then GLA is the stochastic process on $M$ that is defined by
\begin{equation}\label{GLA}
x_{k+1}=\Exp_{x_k}(\epsilon F+\sqrt{2\epsilon g^{-1}}\xi_0)
\end{equation} 
where $F=(F_1,...,F_n)$ with 
\begin{equation}\label{SDE:F}
F_i=-\sum_jg^{ij}\frac{\partial f}{\partial x_j}+\frac{1}{\sqrt{\abs{g}}}\sum_j\frac{\partial}{\partial x_j}\left(\sqrt{\abs{g}}g^{ij}\right),
\end{equation}
$\epsilon>0$ is the stepsize, $\xi_0\sim \mathcal{N}(0,I)$ is the standard Gaussian noise, and $\Exp_x(\cdot):T_xM\rightarrow M$ is the exponential map (Definition \ref{exp map}). Clearly GLA is a two-step discretization scheme of the Riemannian Langevin equation
\[
dX_t=F(X_t)dt+\sqrt{2g^{-1}}dB_t
\]
where $F$ is given by (\ref{SDE:F}). Suppose the position at time $k$ is $x_k$, then the next position $x_{k+1}$ can be obtained by the following tangent-geodesic composition:
\begin{enumerate}
\item Tangent step: Take a local coordinate chart $\varphi:U_{x_k}\rightarrow\mathbb{R}^n$ at $x_k$, this map induces the expression of $g_{ij}$ and $g^{ij}$, then compute the vector $v=\epsilon F+\sqrt{2\epsilon g^{-1}}\xi_0$ in tangent space $T_{x_k}M$;
\item Geodesic step: Solve the geodesic equation (a second order ODE) whose solution is a curve $\gamma(t)\subset\varphi(U_{x_k})$, such that the initial conditions satisfy $\gamma(0)=\varphi(x_k)$ and $\gamma'(0)=v$. Then let $x_{k+1}=\gamma(1)$ be the updated point.
\end{enumerate}
The exponential map and ODE solver for geodesic equations is commonly used in sampling algorithms on manifold, e.g. \cite{VempalaLee17, BG13, LZS16}. We will discuss on other approximations of the exponential map without solving ODEs through illustrations in a later section. Figure \ref{fig1} gives an intuition of GLA on the unit sphere where the exponential map is $\Exp_x(v)=\cos(\norm{v})x+\sin(\norm{v})\frac{v}{\norm{v}}$.
\begin{figure}[h]
\centering
\includegraphics[width=0.5\textwidth]{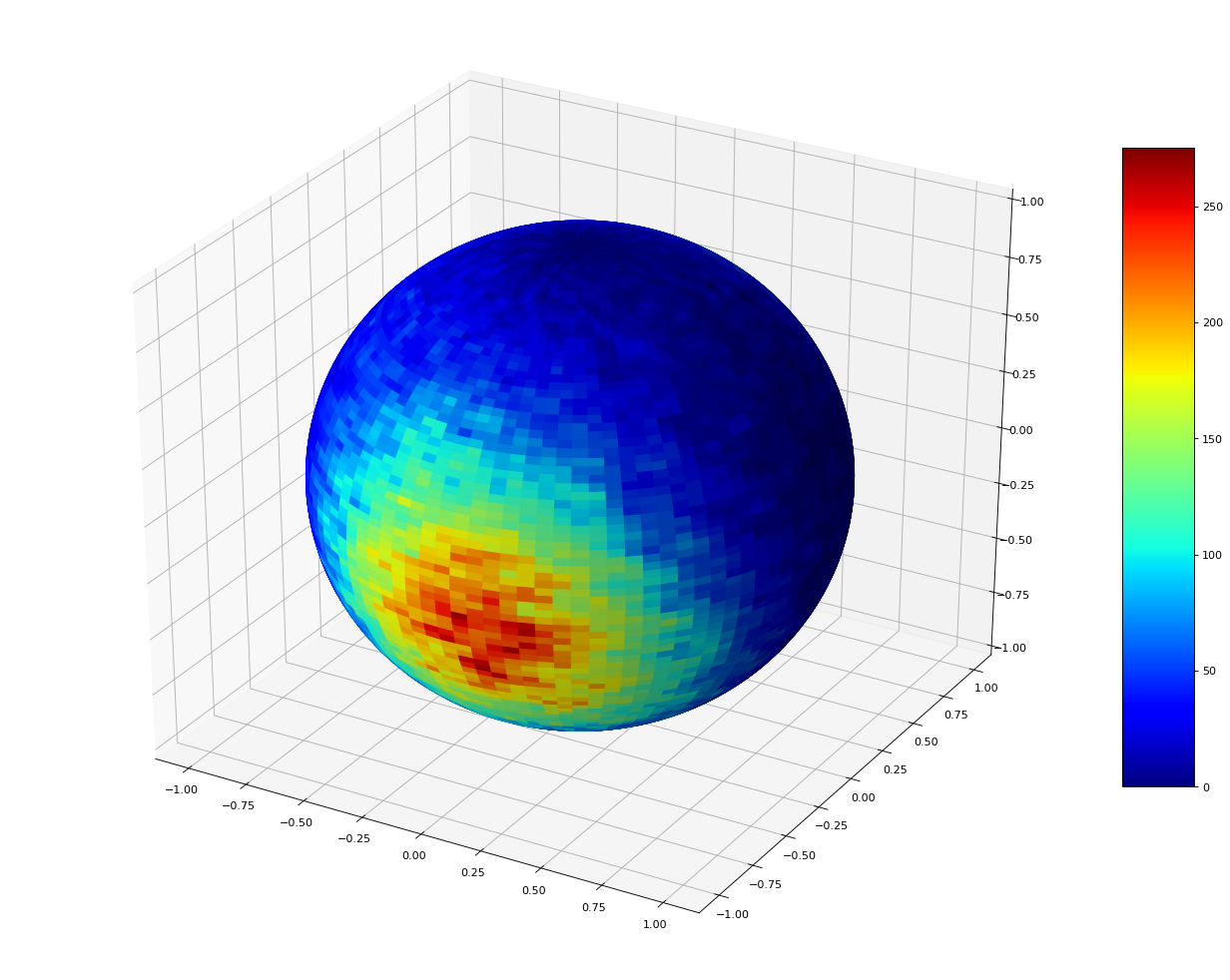}
\caption{$f(x_1,x_2,x_3)=x_1^2+3.05x_2^2-0.9x_3^2+1.1x_1x_2+-1.02x_2x_3+2.1x_3x_1$, $\epsilon=0.1$, Iterations: 100k.}
\label{fig1}
\end{figure}

The main result on convergence is stated as follows.
\begin{theorem}[Informal]\label{mtheoremintro}
Let $M$ be a closed $n$-dimensional manifold (Definition \ref{closedM}). Suppose that $\nu=e^{-f(x)}$ is a distribution on $M$ with $\alpha>0$ the log-Sobolev constant. 
Then there exists a real number $K_2,K_3,K_4,C$, such that
by choosing stepsize $\epsilon$ properly based on the Lipschitz constant of the Riemannian gradient of $f$, log-Sobolev constant of the target distribution $\nu$, dimension and curvature of $M$, the KL divergence $H(\rho_k|\nu)$ decreases along the GLA iterations rapidly in the sense that
\[
H(p_k|\nu)\le e^{-\alpha k\epsilon}H(p_0|\nu)+\frac{16\epsilon}{3\alpha}(2nL^2+2n^3K_2C+nK_3K_4).
\]
\end{theorem}

The same as unadjusted Langevin algorithm (ULA) in Euclidean space, GLA is a biased algorithm that converges to a distribution different from $e^{-f(x)}$. But from the above theorem, one can set up the error $\delta>0$ and take stepsize $\epsilon<\frac{\alpha\delta}{2CnL^2}$ while satisfying the standard from above theorem, and GLA will reach the error $H(\rho_k|\nu)<\delta$ after $k> \frac{2}{\alpha\epsilon}\log\frac{2H(\rho_0|\nu)}{\delta}$ iterations. Practically we need a lower bound estimate for $\alpha$. With additional condition on Ricci curvature, this lower bound can be chosen based on the diameter of $M$ by Theorem \ref{logsobconst}.

Our main technical contributions are:
\begin{itemize}
\item A non-asymptotic convergence guarantee for Geodesic Langevin algorithm on closed manifold is provided, with the help of log-Sobolev inequality. 
\item The framework of this paper serves as the first step understanding to the rate of convergence in sampling from distributions on manifold with log-Sobolev inequality, and can be generalized to prove non-asymptotic convergence results for more general settings and more subtle algorithms, i.e., for open manifolds and unbiased algorithms.
\end{itemize}

\textbf{Comparison to literarture} The typical difference between algorithm (\ref{GLA}) and the classic RLA is the use of exponential map. As $\epsilon\rightarrow 0$, both GLA and RLA boil down to the same continuous time Langevin equation in the local coordinate system:
\[
dX_t=F(X_t)dt+\sqrt{2g^{-1}}dB_t
\]
where $F$ is given by (\ref{SDE:F}) and $B_t$ is the standard Brownian motion in $\mathbb{R}^n$. The direct Euler-Maryuyama discretization iterates in the way that $x_{k+1}=x_t+\epsilon F(x_k)+\sqrt{2\epsilon g^{-1}(x_k)}\xi_0$. However, by adding the vector $\epsilon F(x_k)+\sqrt{2\epsilon g^{-1}(x_k)}\xi_0$ that is in the tangent space to a point $x_k$ that is on the manifold $M$ has no intrinsic geometric meaning, since the resulted point $x_{k+1}$ is indeed not in $M$. The exponential map just gives a way to pull $x_{k+1}$ back to $M$. On the other hand, since RLA is firstly used to sample from distributions on $\mathbb{R}^n$ (or its domain) with a Riemannian structure, \cite{RS02, GC11, PattersonTeh, SDRLS18},
this requires a global coordinate system of $M$, i.e. $M$ is covered by a single coordinate chart and the iterations do not transit between different charts. But this makes it difficult to use RLA when there are inevitably multiple coordinate charts on $M$. More sophisticated algorithms like Geodesic MCMC \cite{LZS16} is used to transit between different coordinate charts, but to the best knowledge of the authors, the rate of convergence is missing in the literature. Li and Erdogdu \cite{LiErdogdu} generalize the result of \cite{VW19} by implementing the Riemannian Langevin algorithm in two steps (gradient+Riemannian Brownian motion).

\section{Related Works}
 Unadjusted Langevin algorithm (ULA) when sampling from a strongly logconcave density in  Euclidean space has been studied extensively in the literature. The bounds for ULA is known in \cite{CB, ArnakD, DalalyanKara, DMM}. 
 The case when $f$ is strongly convex and has Lipschitz gradient is studied by \cite{ArnakD17, DurmusMouline17, DurmusMouline19}. 
 Since ULA is biased because of the discretization, i.e. it converges to a limit distribution that is different from that from continuous Langevin equation. the Metropolis-Hastings correction is widely used to correct this bias, e.g. \cite{RT96, DCWY18}.
  A simplified correction algorithm is proposed by \cite{Wibisono18} that is called symmetrized Langevin algorithm with a smaller bias than ULA.
 Convergence results is obtained for Proximal Langevin algorithm (PLA) in \cite{Wibisono19}. In the case where the target distribution is log-concave, there are other algorithms proven to converge rapidly, i.e., Langevin Monte Carlo by \cite{Bernton18}, 
 ball walk and hit-and-run \cite{KLS97, LV06, LV07, LVempala06}, and Hamiltonian Monte Carlo by \cite{DMS17, VempalaLee18, MVishnoi18}. 
 The underdamped version of the Langevin dynamics under log-Sobolev inequality is studied by \cite{MaCCFBJ}, where an iteration complexity for the discrete time algorithm that has better dependence on the dimension is provided. A coupling approach is used by \cite{EGZ18} to quantify convergence to equilibrium for Langevin dynamics that yields contractions in a particular Wasserstein distance and provides precise bounds for convergence to equilibrium. 
 The case where the densities that are neither smooth nor log-concave is studied in \cite{LuuFC} and asymptotic consistency guarantees is provided. 
 For the Wasserstein distance, \cite{CCYBJ, MMSz, RRT17} provide convergence bound. An earlier research on stochastic gradient Langevin dynamics with application in Bayesian learning is proposed by \cite{Welling}, The Langevin Monte Carlo with a weaker smoothness assumption is studied by \cite{CDJB19}. In order to improve sample quality, \cite{GMackey} develops a theory of weak convergence for kernel Stein discrepancy based on Stein's method. 
   In general, sampling from non log-concave densities is hard, \cite{GeLR} gives an exponential lower bound on the number of queries required. 

The Riemannian Langevin algorithm has been studied in different extent. Related to volume computation of a convex body in Euclidean space, one can endow the interior of a convex body the structure of a Hessian manifold and run geodesic (with respect to the Hessian metric) random walk \cite{VempalaLee17} that is a discretization scheme of a stochastic process with uniform measure as the stationary distribution. The rigorous proof of the convergence of Riemannian Hamiltonian Monte Carlo for sampling Gibbs distribution and uniform distribution in a polytope is given by \cite{VempalaLee18}. In sampling non-uniform distribution, \cite{ZPFP} gives a discretization scheme related to mirror descent and a non-asymptotic upper bound on the sampling error of the Riemannian Langevin Monte Carlo algorithm in Hessian manifold. The mirrored Langevin is firstly considered by \cite{HKRC18} and a non-asymptotic rate is obtained and generalized to the case when only stochastic gradients (mini-batch) are available. 
An affine invariant perspective of continuous time Langevin dynamics for Bayesian inference is studied in \cite{INR19}. Positive curvature is used to show concentration results for Hamiltonian Monte Carlo in \cite{SRSH}. \cite{LZZ19} understand MCMC as gradient flows on Wasserstein spaces and HMC on implicitly defined manifolds is studied in \cite{BSU}.

\section{Preliminaries}
For a complete introduction to Riemannian manifold and stochastic analysis on manifold, we recommend \cite{JLee} and \cite{Hsu} for references.
\subsection{Riemannian geometry}
\begin{definition}[Manifold]
A $C^k$-differentiable, $n$-dimensional manifold is a topological space $M$, together with a collection of coordinate charts $\{(U_{\alpha},\varphi_{\alpha})\}$, where each $\varphi_{\alpha}$ is a $C^k$-diffeomorphism from an open subset $U_{\alpha}\subset M$ to $\mathbb{R}^n$. The charts are compatible in the sense that, whenever $U_{\alpha}\cap U_{\beta}\ne\emptyset$, the transition map $\varphi_{\alpha}\circ\varphi_{\beta}^{-1}(U_{\beta}\cap U_{\alpha})\rightarrow\mathbb{R}^n$ is of $C^k$.
\end{definition}

\begin{definition}[Closed manifold]\label{closedM}
A manifold $M$ is called \emph{closed} if $M$ is compact and has no boundary.
\end{definition}
Typical examples of closed manifolds include sphere and torus.

\begin{definition}[Riemannian metric]\label{rmetric}
A Riemannian manifold $(M,g)$ is a differentiable manifold $M$ with a Riemannian metric $g$ defined as the inner product on the tangent space $T_xM$ for each point $x$, $g(\cdot,\cdot):T_xM\times T_xM\rightarrow\mathbb{R}$. Then length of a smooth path $\gamma:[0,1]\rightarrow M$ is $\abs{\gamma}=\int_0^1\sqrt{g(\gamma'(t),\gamma'(t))}dt$. In a local coordinate chart, $g$ is represented by a $n\times n$ symmetric positive definite matrix with entries $g_{ij}$.
\end{definition}
\begin{definition}[Geodesic]\label{geodesic}
We call a curve $\gamma(t):[0,1]\rightarrow M$ a geodesic if it satisfies both of the following conditions:
\begin{enumerate}
\item The curve $\gamma(t)$ is parametrized with constant speed, i.e. $\norm{\frac{d}{dt}\gamma(t)}_{\gamma(t)}$ is constant for $t\in[0,1]$.
\item The curve is the locally shortest length curve between $\gamma(0)$ and $\gamma(1)$, i.e. for any family of curve $c(t,s)$ with $c(t,0)=\gamma(t)$ and $c(0,s)=\gamma(0)$ and $c(1,s)=\gamma(1)$, we have $\frac{d}{ds}|_{s=0}\int_{0}^1\norm{\frac{d}{dt}c(t,s)}_{c(t,s)}dt=0$. 
\end{enumerate}
\end{definition}
We use $\gamma_{x\rightarrow y}$ to denote the geodesic from $x$ to $y$ ($\gamma_{x\rightarrow y}(0)=x$ and $\gamma_{x\rightarrow y}(1)=y$).
The most important property of a geodesic $\gamma(t)$ is that the time derivative $\dot{\gamma}(t)$ as a vector field, has 0 covariant derivative, i.e. $\nabla_{\dot{\gamma}(t)}\dot{\gamma}(t)=0$. This property boils down to a second order ODE in local coordinate systems, 
\[
\ddot{\gamma}_i(t)+\sum_{j,k}\Gamma_{jk}^i\left(\gamma(t)\right)\dot{\gamma}_j(t)\dot{\gamma}_k(t)=0
\]  
for $i\in[n]$, where $\Gamma_{jk}^i$ are the Christoffel symbols. Given a initial position $\gamma(0)$ and initial velocity $\dot{\gamma}(0)$, by the fundamental theorem of ODE, there exists a unique solution satisfying the geodesic equation. This is the principle we can use  the ODE solver in GLA.

\begin{definition}[Exponential map]\label{exp map}
The exponential map $\Exp_x(v)$ is maps $v\in T_xM$ to $y\in M$ such that there exists a geodesic $\gamma$ with $\gamma(0)=x$, $\gamma(1)=y$ and $\gamma'(0)=v$. 
\end{definition}
The exponential map can be thought of moving a point along a vector in manifold in the sense that the exponential map in $\mathbb{R}^n$ is nothing but $\Exp_x(v)=x+v$. The  exponential map on sphere at $x$ with direction $v$ is $\Exp_x(v)=\cos(\norm{v})x+\sin(\norm{v})\frac{v}{\norm{v}}$.

\begin{definition}[Parallel transport]\label{paralleltransport}
The parallel transport $\Gamma_x^y$ is a map that transport $v\in T_xM$ to $\Gamma_x^y\in T_xM$ along $\gamma_{x\rightarrow y}$ such that the vector stays constant by satisfying a zero-acceleration condition.
\end{definition}

Next, we refer the definition of Riemannian gradient and divergence only in local coordinate systems that is used in this paper.
\begin{definition}[Gradient and Divergence]
In local coordinate system, the gradient of $f$ and the divergence of a vector field $V=\sum_iV_i\frac{\partial}{\partial x_i}$ on a Riemannian manifold is given by
\[
\grad f=\sum_{i,j}g^{ij}\frac{\partial f}{\partial x_i}\frac{\partial}{\partial x_i} \ \ \ \text{and}\ \ \ \Div V=\frac{1}{\sqrt{\abs{g}}}\sum_{i}\frac{\partial}{\partial x_i}\left(\sqrt{\abs{g}}V_i\right)
\]
where $g^{ij}$ is the $ij$-th entry of the inverse matrix $g^{-1}$ of $g$, $\abs{g}=\det(g_{ij})$.
\end{definition}
\begin{definition}[Lipschitz gradient]\label{Lipschitz}
$f$ is of Lipschitz gradient if there exists a constant $L>0$ such that
\[
\norm{\grad f(y)-\Gamma_x^y\grad f(x)}\le Ld(x,y) \ \ \text{for all}\ \ x,y\in M
\]
where $d(x,y)$ is the geodesic distance between $x$ and $y$, and $\Gamma_x^y$ is the parallel transport from $x$ to $y$, see Definition \ref{paralleltransport}.
\end{definition}

\subsection{Stochastic differential equations}
Let $\{X_t\}_{t\ge 0}$ be a stochastic process in $\mathbb{R}^n$ and $B_t$ be the standard Brownian motion in $\mathbb{R}^n$.
\\
\\
\textbf{Fokker-Planck Equation} For any stochastic differential equation of the form
\[
dX_t=F(X_t,t)dt+\sigma(X_t,t)dB_t,
\]
the probability density of the SDE is given by the PDE
\[
\frac{\partial\rho(x,t)}{\partial t}=-\sum_{i=1}^n\frac{\partial}{\partial x_i}\left(F_i(x,t)\rho(x,t)\right)+\sum_{i=1}^n\sum_{j=1}^n\frac{\partial^2}{\partial x_i\partial x_j}\left(A_{ij}(x,t)\rho(x,t)\right)
\]
where $A=\frac{1}{2}\sigma\sigma^{\top}$, i.e. $A_{ij}=\frac{1}{2}\sum_{k=1}^n\sigma_{ik}(x,t)\sigma_{jk}(x,t)$

\subsection{Distributions on manifold}
Let $\rho$ and $\nu$ be probability distribution on $M$ that is absolutely continuous with respect to the Riemannian volume measure (denoted by $dx$) on $M$.

\begin{definition}[KL divergence]
The Kullback-Leibler (KL) divergence of $\rho$ with respect to $\nu$ is
\[
H(\rho|\nu)=\int_M\rho(x)\log\frac{\rho(x)}{\nu(x)}dx.
\]
\end{definition}

\begin{definition}[Wasserstein distance]
The Wasserstein distance between $\mu$ and $\nu$ is defined to be
\[
W_2(\mu,\nu)=\inf\{\sqrt{\mathbb{E}[d(X,Y)^2]}:\text{law}(X)=\mu,\text{law}(Y)=\nu\}.
\]
\end{definition}

\begin{definition}[Talagrand inequality]\label{Tala}
The probability measure $\nu$ satisfies a Talagrand inequality with constant $\alpha>0$ if for all probability measure $\rho$, absolutely continuous with respect to $\nu$, with finite moments of order 2,
\[
W_2(\rho,\nu)^2\le\frac{2}{\alpha}H(\rho|\nu)
\]
\end{definition}

\begin{definition}[Log-Sobolev inequality]
A probability measure $\nu$ on $M$ is called to satisfy the logarithmic Sobolev inequality (LSI) if there exists a constant $\alpha>0$ such that
\[
\int_Mg^2\log g^2d\nu-\left(\int_Mg^2d\nu\right)\log\left(\int_Mg^2d\nu\right)\le\frac{2}{\alpha}\int_M\norm{\grad g}^2d\nu,
\]
for all smooth functions $g:M\rightarrow\mathbb{R}$ with $\int_Mg^2\le\infty$. The largest possible constant $\alpha$ is called the logarithmic Sobolev constant (LSC).
\end{definition}

\paragraph{Estimate of the log-Sobolev constant} It is well known that a compact manifold always satisfies log-Sobolev inequality \cite{LGross, OV, Rothaus81, Rothaus86, Wang1997, Wang1997x}.
Practically we need a specific lower bound $\alpha_0$ of the log-Sobolev constant $\alpha$, so that we can choose stepsize $\epsilon\le\frac{\alpha_0\delta}{2CnL^2}\le\frac{\alpha\delta}{2CnL^2}$ for a given error bound $\delta$ for the KL divergence, see the discussion next to Theorem \ref{mtheoremintro}. The estimate of $\alpha_0$ is closely related to the Ricci curvature of  $M$ and the first eigenvalue $\lambda_1$ of the Laplace-Beltrami operator on $M$, see the definitions in Appendix. For the case where $M$ is compact and the Ricci curvature is non-negative, the lower bound estimate for $\alpha$ is clear from the following theorem.
\begin{theorem}[Theorem 7.3, \cite{MichelLedoux}]\label{logsobconst}
Let $M$ be a compact Riemannian manifold with diameter $D$ and non-negative Ricci curvature. Then the log-Sobolev constant $\alpha$ satisfies
$
\alpha\ge\frac{\lambda_1}{1+2D\sqrt{\lambda_1}}$.
In particular,
$
\alpha\ge\frac{\pi^2}{(1+2\pi)D^2}$.
\end{theorem}
The last inequality shows that we can choose a lower bound $\alpha_0=\frac{\pi^2}{(1+2\pi)D^2}$ that only depends on the diameter of $M$. For more results on estimate of the log-Sobolev constant, we refer to \cite{Wang1997, Wang1997x}.

\section{Main Results}\label{MainResults}
\subsection{Technical Overview}
\paragraph{Wasserstein gradient flow.}
The equivalence between Langevin dynamics and optimization in the space of densities is based on the result of \cite{JKO, Wibisono18} that the Langevin dynamics captures the gradient flow of the relative entropy functional in the space of densities with the Wasserstein metric. As a result, running the Langevin dynamics is equivalent to sampling from the stationary distribution of the Wasserstein gradient flow asymptotically. To minimize $\int_Mf(x)\rho(x)dx$ with respect to $\rho\in\mathcal{P}(M)$, we consider the entropy regularized functional of $\rho$ defined as follows,
\[
\mathcal{L}(\rho)=\mathcal{F}(\rho)+H(\rho)
\]
where $\mathcal{F}(\rho)=\int_Mf(x)\rho(x)dx$ and $H(\rho)=\int_M\rho(x)\log \rho(x)dx$ that is the negative Shannon entropy $h(\rho)=-\int_M\rho(x)\log \rho(x)dx$. According to \cite{FSant}, the Wasserstein gradient flow associated with functional $\mathcal{L}$ is the Fokker-Planck equation
\begin{equation}\label{FP00}
\frac{\partial\rho(x,t)}{\partial t}=\Div\left(\rho(x,t)\grad f(x)+\grad \rho(x,t)\right)=\Div\left(\rho(x,t)\grad f(x)\right)+\Delta_M\rho(x,t),
\end{equation}
where $\grad$ and $\Div$ are gradient and divergence in Riemannian manifold, and $\Delta_M$ is the Laplace-Beltrami operator generalizing the Euclidean Laplacian $\Delta$ to Riemannian manifold. More details can be found in Appendix. The stationary solution of equation (\ref{FP00}) is $e^{-f(x)}$ that minimizes the entropy regularized functional $\mathcal{L}(\rho)$, and then the optimization problem over the space of densities boils down to track the evolution of $\rho(x,t)$ that is defined by equation (\ref{FP00}). 
\paragraph{Coordinate-independent Langevin equation.}
In order to implement the aforementioned evolution of $\rho(x,t)$ in Euclidean space, one can simulate the stochastic process $\{X_t\}_{t\ge 0}$ defined by the Langevin equation:
$dX_t=-\nabla f(X_t)dt+\sqrt{2}dB_t$, 
where $B_t$ is the standard Brownian motion and $X_t$ has $\rho(x,t)$ as its density function. In contrast to the Euclidean case, we need a coordinate-independent formulation of Langevin equation, e.g. \cite{BKR}. This is derived by expanding the Fokker-Planck equation (\ref{FP00}) in a local coordinate system and is written in the following form:
\begin{equation}\label{FP:local}
dX_t=F(X_t)dt+\sqrt{2g^{-1}}dB_t
\end{equation}
where $F(X_t)$ is a vector with $i$'th component $F_i=-\sum_jg^{ij}\frac{\partial f}{\partial x_j}+\frac{1}{\sqrt{\abs{g}}}\sum_j\frac{\partial}{\partial x_j}\left(\sqrt{\abs{g}}g^{ij}\right)$ and $\abs{g}$ is the determinant of metric matrix $g_{ij}$. Note that this local form indicate the fact from Fokker-Planck equation that the process $\{X_t\}_{t\ge 0}$ is the negative gradient of $f$ followed by a manifold Brownian motion. The rate of convergence we are interested in is the classic Euler-Maryuyama discretization scheme in manifold setting, i.e. compute the vector in tangent space and project it onto the base manifold through exponential map. So the discretization error consists of two parts: one is from considering $\grad f(x_t)$ as constant in a neighborhood of $x_t$, and the other is from the approximation of a curved neighborhood of $x_t$ with the tangent space at $x_t$. 
The main task in the proofs of Theorem \ref{convergence:discrete} and \ref{convergence:constant curvature} is to bound the aforementioned two parts of errors and compare with the density evolving along continuous time Langevin equation.

\subsection{Convergence Analysis}
We state some assumptions before presenting main theorems.
\begin{assumption}\label{ass1}
$M$ is a closed manifold. 
\end{assumption}
It means $M$ is compact and has no boundary, Definition \ref{closedM}. This assumption is essentially used to make the boundary integral on $\partial M$, i.e. $\int_{\partial M}\log\frac{\rho_t}{\nu}\langle\rho_t\grad f+\grad\rho_t,\vec{n}\rangle dx$ in the proof of Lemma \ref{lemma:continuous1}, to vanish, see Appendix. This assumption can be relaxed to open manifold by assuming the integral decreases fast as $x$ approaches the infinity.
\begin{assumption}\label{ass2}
$f(x)$ is differentiable on $M$. 
\end{assumption}
An immediate consequence by combining Assumption \ref{ass1} and \ref{ass2} is that there exists a number $L>0$, such that the Riemannian gradient $\grad f$ of $f$ is $L$-Lipschitz (Definition \ref{Lipschitz}) due to the compactness of $M$. Another crucial property used in the proof is that the target distribution $e^{-f}$ satisfies the log-Sobolev inequality, and this can also be derived by compactness of $M$. 

Since the prerequisite of convergence of GLA is the convergence of the continuous time Langevin equation, we show the KL divergence between $\rho_t$ and $\nu$ converges along the continuous time Riemannian Langevin equation. The proof is completed by the following lemma showing that $H(\rho_t|\nu)$ decreases since $\frac{d}{dt}H(\rho_t|\nu)< 0$ for all $\rho_t\ne\nu$. 
Based on the analysis of the previous section, it suffices to track the evolution of $\rho_t$ according to the Fokker-Planck equation (\ref{FP00}).
\begin{lemma}\label{lemma:continuous1}
Suppose $\rho_t$ evolves following the Fokker-Planck equation (\ref{FP00}), then
\[
\frac{d}{dt}H(\rho_t|\nu)=-\int_M\rho_t(x)\norm{\grad\log\frac{\rho_t(x)}{\nu(x)}}^2dx
\]
where $dx$ is the Riemannian volume element.
\end{lemma}
The proof is a straightforward calculation of the time derivative of $H(\rho_t|\nu)$, followed by the expression of $\frac{\partial \rho_t}{\partial t}$ in equation (\ref{FP00}), i.e.
\begin{align}
\frac{d}{dt}H(\rho_t|\nu)&=\int_M\frac{\partial\rho_t}{\partial t}\log\frac{\rho_t}{\nu}dx=\int_M\Div(\rho_t\grad f+\grad\rho_t)\log\frac{\rho_t}{\nu}dx
\\
&=-\int_M\rho_t\norm{\grad \log\frac{\rho_t}{\nu}}^2dx+\int_{\partial M}\log\frac{\rho_t}{\nu}\langle\rho_t\grad f+\grad\rho_t,\vec{n}\rangle dx
\end{align}
The result follows from integration by parts and the Assumption \ref{ass1}. Details are left in Appendix.

Since $M$ is compact, there exists a constant $\alpha>0$ such that the log-Sobolev inequality (LSI) holds. So we can get the following convergence of KL divergence for continuous Langevin dynamics immediately.
\begin{theorem}\label{convergence:continuous}
Suppose $\nu$ satisfies LSI with constant $\alpha>0$. Then along the Riemannian Langevin equation, i.e. the SDE (\ref{FP:local}) in local coordinate systems, the density $\rho_t$ satisfies
\[
H(\rho_t|\nu)\le e^{-2\alpha t}H(\rho_0|\nu).
\]
\end{theorem}

The following theorem shows that the KL divergence $H(\rho_k|\nu)$ decreases geometrically along the GLA dynamics.
\begin{theorem}\label{convergence:discrete}
Suppose $M$ is a compact manifold without boundary and $R$ is the Riemann curvature, $\nu=e^{-f}$ a density on $M$ with $\alpha>0$ the log-Sobolev constant. Then there exists a global constant $K_2,K_3,K_4,C$, such that for any $x_0\sim\rho_0$ with $H(\rho_0|\nu)\le\infty$, the iterates $x_k\sim\rho_k$ of GLA with stepsize $\epsilon\le\min\{\frac{\alpha}{4L\sqrt{L^2+n^2K_2C}},\frac{2nL^2+2n^3K_2C+nK_3K_4}{2(nL^3+n^3K_2CL)},\frac{1}{2L},\frac{1}{2\alpha}\}$ satisfty

\[
H(p_k|\nu)\le e^{-\alpha k\epsilon}H(p_0|\nu)+\frac{16\epsilon}{3\alpha}(2nL^2+2n^3K_2C+nK_3K_4)
\]
\end{theorem}
The convergence of KL divergence implies the convergence of Wasserstein distance. 
\begin{proposition}
For the closed manifold $M$ with a density $\nu=e^{-f(x)}$, the iterates $x_k\sim\rho_k$ of GLA with a properly chosen stepsize satisfy
\[
W_2(\rho_k,\nu)^2\le\frac{2}{\alpha}e^{-\alpha k\epsilon}H(p_0|\nu)+\frac{32\epsilon}{3\alpha^2}(2nL^2+2n^3K_2C+nK_3K_4).
\]
\end{proposition}

\begin{proof}
It is an immediate consequence from the convergence of KL divergence that the Wasserstein distance $W_2(\rho_k,\nu)$ converges rabidly since log-Sobolev inequality implies Talagrand inequality (\cite{Talagrand, OV}). Since $\nu$ satisfies log-Sobolev inequality, then we have
\begin{align}
W_2(\rho_k|\nu)^2&\le\frac{2}{\alpha}H(\rho_k|\nu)
\\
&\le\frac{2}{\alpha}\left(e^{-\alpha k\epsilon}H(p_0|\nu)+\frac{16\epsilon}{3\alpha}(2nL^2+2n^3K_2C+nK_3K_4)\right)
\\
&=\frac{2}{\alpha}e^{-\alpha k\epsilon}H(p_0|\nu)+\frac{32\epsilon}{3\alpha^2}(2nL^2+2n^3K_2C+nK_3K_4)
\end{align}
and the proof completes.
\end{proof}

\section{Conclusion}
In this paper we focus on the problem of sampling from a distribution on a Riemannian manifold and propose the Geodesic Langevin Algorithm. GLA modifies the Riemannian Langevin algorithm by using exponential map so that the algorithm is defined globally. By leveraging the geometric meaning of GLA, we provide a non-asymptotic convergence guarantee in the sense that the KL divergence (as well as the Wasserstein distance) decreases fast along the iterations of GLA. By assuming that we have full access to the geometric data of the manifold, we can control the bias between the stationary distribution of GLA and the target distribution to be arbitrarily small through the choice of stepsize. The assumptions on the joint densities are not natural and there is no obvious way to determine the constants. Further work is expected to improve the results so that they do not depend on the assumption.

\section*{Acknowledgement}
We thank Mufan (Bill) Li and Murat A. Erdogdu for pointing out the mistakes in the original version and their helpful comments on revision and correction.
Xiao Wang would like to acknowledge the NRF-NRFFAI1-2019-0003, SRG ISTD 2018 136 and NRF2019-NRF-ANR095 ALIAS grant. Qi Lei is supported by Computing Innovation Fellowship.

\bibliography{Bibli}
\bibliographystyle{plain}
\newpage
\appendix

\section{More background}
\subsection{Calculus on manifold}

\begin{definition}[Levi-Civita Connection]
Let $(M,g)$ be a Riemannian manifold. An affine connection is said to be the Levi-Civita connection if it is torsion-free. i.e.
\[
\nabla_XY-\nabla_YX=[X,Y]
\]
for every pair of vector fields $X,Y$ on $M$ and preserves the metric i.e.
\[
\nabla g=0.
\]
\end{definition}

\begin{definition}[Riemannian Volume]
Let $(M,g)$ be an orientable Riemannian manifold. The volume form on the manifold in local coordinates is given as
\[
d\vol=\sqrt{\det (g)}dx_1\wedge...\wedge dx_n.
\]
\end{definition}
We denote $\abs{g}=\det(g)$ and $dx=d\vol$ (if no ambiguities caused) for short throughout following context.
\\
The following Theorem is used to guarantee the exponential map is defined on the whole tangent space, which is equivalent to require $M$ to be complete. This property is satisfied in our setting for $M$ to be compact without boundary.
\begin{theorem}[Hopf-Rinow]
Let $(M,g)$ be a connected Riemannian manifold. Then the followings are equivalent.
\begin{enumerate}
\item The closed and bounded subsets of $M$ are compact.
\item $M$ is a complete metric space.
\item $M$ is geodescically complete: for every point $x\in M$, the exponential map $\Exp_x$ is defined on the entire tangent space $T_xM$.
\end{enumerate}
\end{theorem}

The notion of differential operators, e.g. gradient, divergence and Laplacian for the differentiable functions and vector fields on Euclidean space can be generalized to Riemannian manifold. In local coordinate system, $\{\partial_i=\frac{\partial}{\partial x_i}:i\in[n]\}$ is a basis of the tangent space $T_xM$. Denote $g_{ij}$ the metric matrix, $g^{ij}$ the inverse of $g_{ij}$ and $\abs{g}=\det g_{ij}$ the determinant of matrix $g_{ij}$. Let $f$ and $V$ be differentiable function and vector field on $M$, then the Riemannian gradient of $f$ and the divergence of $V$ are written as
\[
\grad f=\sum_{i,j}g^{ij}\frac{\partial f}{\partial x_i}\partial_i \ \ \ \text{and}\ \ \ \Div V=\frac{1}{\sqrt{\abs{g}}}\sum_{i}\frac{\partial}{\partial x_i}\left(\sqrt{\abs{g}}V_i\right)
\] 
where $V_i$ is the $i$-th component of $V$.

The Laplace-Beltrami operator $\Delta_M$ acting on $f$ is defined to be the divergence of the gradient of $f$, i.e.
\[
\Delta_Mf=\Div(\grad f)=\frac{1}{\sqrt{\abs{g}}}\sum_i\frac{\partial}{\partial x_i}\left(\sqrt{\abs{g}}\sum_{j}g^{ij}\frac{\partial f}{\partial x_j}\right).
\]
In Euclidean space, $\Delta_M$ boils down to the classic Laplacian $\Delta f=\nabla\cdot(\nabla f)$.
\\
The following integration by parts formulas are used in proof of main lemmas. 
Let $M$ be a compact oriented Riemannian manifold of dimension $n$ with boundary $\partial M$. Let $X$ be a vector field on $M$. The integration by parts is given by
\[
\int_M\langle\grad f,X\rangle=-\int_M f\Div X+\int_{\partial M}f\langle X,n\rangle
\]
or Green's formula
\[
\int_M(f\Delta_M g-g\Delta_M f)=-\int_{\partial M}\left(f\frac{\partial g}{\partial n}-g\frac{\partial f}{\partial n}\right)
\]

If $\partial M$ is empty or the vector field $X$ decay sufficiently fast at infinity of $M$ provided $M$ is open, we have
\[
\int_M\langle\grad f,X\rangle=-\int_Mf\Div X.
\]

\begin{definition}[First eigenvalue of Laplacian]\label{eigenLaplacian}
The first eigenvalue $\lambda_1\ge 0$ of the Laplacian operator on $M$ is defined to be
\[
\lambda_1=\inf_{f\in C_c^{\infty}}\Big\lbrace\frac{\int_M\norm{\grad f}^2dx}{\int_M\norm{f}^2dx}\Big\rbrace.
\]
\end{definition}

\subsection{Stochastic analysis on manifold}
Recall that the standard Brownian motion in $\mathbb{R}^n$ is a random process $\{X_t\}_{t\ge 0}$ whose density evolves according to the diffusion equation 
\[
\frac{\partial\rho(x,t)}{\partial t}=\frac{1}{2}\Delta\rho(x,t).
\]
 Similarly, the Brownian motion in manifold $M$ is $M$-valued random process $\{W_t\}_{t\ge 0}$ whose density function evolves according to the diffusion equation with respect to Laplace-Beltrami operator which is the counterpart of the Laplace operator on Euclidean space.
 \[
 \frac{\partial\rho(x,t)}{\partial t}=\frac{1}{2}\Delta_M\rho(x,t).
 \]
 In local coordinate, the Laplace-Beltrami is written as
 \[
 \Delta_M=\sum_{i.j}g^{ij}\frac{\partial^2}{\partial x_i\partial x_j}+\sum_ib_i\frac{\partial}{\partial x_i},
 \]
 where 
 \begin{equation}\label{b}
 b_i=\sum_j\frac{1}{\sqrt{\abs{g}}}\frac{\partial}{\partial x_j}\left(\sqrt{\abs{g}}g^{ij}\right)=\sum_{j,k}g^{jk}\Gamma^i_{jk}.
 \end{equation}
 We can construct Brownian motion in the local coordinate as the solution of the stochastic differential equation for a process $\{X_t\}_{t\ge 0}$:
 \[
 dX_t=\frac{1}{2}b(X_t)dt+\sigma(X_t)dB_t
 \]
 where the component $b_i(X_t)$ of $b(X_t)$ is given by (\ref{b}) and $\sigma=(\sigma_{ij})$ is the unique symmetric square root of $g^{-1}=(g^{ij})$.

\section{Derivation of the GLA}
In this section, we give detailed explanation on that the Riemannian Langevin algorithm, as a stochastic process, captures the dynamics of the evolution of the density function for the stochastic process. The derivation is firstly to write the diffusion equation in local coordinate system of the manifold, and then compare the corresponding terms to the Fokker-Planck equation related to stochastic differential equation that gives insight to the local expression of Riemannian Langevin algorithm. In order to do this, recall that the density $e^{-f}$ on $M$ is the stationary solution of the PDE
\begin{align}
\frac{\partial \rho_t}{\partial t}&=\Div\left(\rho_t\grad f+\grad \rho_t\right).
\end{align}
Using the local expression of Riemannian gradient and divergence operator, this PDE can be written as
\begin{align}
\frac{\partial \rho_t}{\partial t}=&\frac{1}{\sqrt{\abs{g}}}\sum_{i=1}^n\frac{\partial}{\partial x_i}\left(\sqrt{\abs{g}}\left(\sum_j g^{ij}\frac{\partial f}{\partial x_j}\right)\rho_t+\sqrt{\abs{g}}\sum_jg^{ij}\frac{\partial \rho_t}{\partial x_j}\right)
\\
=&\frac{1}{\sqrt{\abs{g}}}\sum_i\frac{\partial}{\partial x_i}\left(\left(\sum_jg^{ij}\frac{\partial f}{\partial x_j}-\frac{1}{\sqrt{\abs{g}}}\sum_j\frac{\partial}{\partial x_j}\left(\sqrt{\abs{g}}g^{ij}\right)\right)\sqrt{\abs{g}}\rho_t\right)
\\
&+\frac{1}{\sqrt{\abs{g}}}\sum_{i,j}\frac{\partial^2}{\partial x_i\partial x_j}\left(g^{ij}\sqrt{\abs{g}}\rho_t\right)
\end{align}
Denoting $\tilde{\rho}_t=\sqrt{\abs{g}}\rho_t$, we have the Fokker-Planck equation of density  in Euclidean space as follows,
\begin{equation}\label{FP1}
\pder[\tilde{\rho}_t]{t}=-\sum_i\pder{x_i}\left(\left(\frac{1}{\sqrt{\abs{g}}}\sum_j\pder{x_j}\left(\sqrt{\abs{g}}g^{ij}\right)-\sum_jg^{ij}\pder[f]{x_j}\right)\tilde{\rho}_t\right)+\sum_{i,j}\frac{\partial^2}{\partial x_i\partial x_j}\left(g^{ij}\tilde{\rho}_t\right).
\end{equation}
Since for any stochastic differential equation of the form
\[
dX_t=F(X_t,t)dt+\sigma(X_t,t)dB_t
\]
the density $p_t$ for $X_t$ satisfies 
\begin{equation}\label{FP2}
\frac{\partial p(x,t)}{\partial t}=-\sum_{i=1}^n\frac{\partial}{\partial x_i}\left(F_i(x,t)p(x,t)\right)+\sum_{i=1}^n\sum_{j=1}^n\frac{\partial^2}{\partial x_i\partial x_j}\left(A_{ij}(x,t)p(x,t)\right)
\end{equation}
where $A=\frac{1}{2}\sigma\sigma^{\top}$, i.e. $A_{ij}=\frac{1}{2}\sum_{k=1}^n\sigma_{ik}(x,t)\sigma_{jk}(x,t)$. Compare equations (\ref{FP1}) and (\ref{FP2}), we have the drift and diffusion terms in local coordinate systems are given by
\[
F_i(x_t)=-\sum_jg^{ij}\frac{\partial f}{\partial x_j}+\frac{1}{\sqrt{\abs{g}}}\sum_j\frac{\partial}{\partial x_j}\left(\sqrt{\abs{g}}g^{ij}\right)
\]
and
\[
\sigma(x_t)=\sqrt{2(A_{ij})}=\sqrt{2(g^{ij})}=\sqrt{2g^{-1}}.
\]
So the local Langevin equation is 
\begin{equation}
dX_t=F(X_t)dt+\sqrt{2g^{-1}}dB_t.
\end{equation}
This equation describes infinitesimal evolution of $X_t$, which can be seen as a process in the tangent space of $M$. The Riemannian Langevin algorithm is the classic Euler-Maruyama discretization in the tangent space, i.e., by letting $X_t$ move in the tangent space for a positive time interval $t\in[0,\epsilon]$ with the drift and diffusion at current location. Suppose the initial point is $x_0$, the tangent vector is
\[
\epsilon F(x_0)+\sqrt{2\epsilon g^{-1}(x_0)}\xi_0
\]
where $\xi_0\sim\mathcal{N}(0,I)$ is the standard Gaussian noise. Then the updated point is obtained by mapping the vector to the base manifold via exponential map,
\[
x_1=\Exp_{x_0}\left(\epsilon F(x_0)+\sqrt{2\epsilon g^{-1}(x_0)}\xi_0\right).
\]
Renaming $x_k=x_0$ and $x_{k+1}=x_1$, we have the general form
\[
x_{k+1}=\Exp_{x_k}\left(\epsilon F(x_k)+\sqrt{2\epsilon g^{-1}(x_k)}\xi_0\right).
\]
We give the expression of the algorithm in normal coordinate, for convenience in part of the proofs of main theorems.

For any manifold $M$, and $x\in M$, $T_xM$ is isomorphic to $\mathbb{R}^n$, $\exp_x^{-1}$ gives a local coordinate system of $M$ around $x$. This is called the normal coordinates at $x$. The following lemmas are from Lee-Vampalar
\begin{lemma}
In normal coordinate, we have
\[
g_{ij}(x)=\delta_{ij}-\frac{1}{3}\sum_{kl}R_{ikjl}(x)x^kx^l+O(\abs{x}^3).
\]
\end{lemma}
Under normal coordinate, the RLA can be written as
\[
x_{t+1}=\Exp_{x_t}(-\epsilon\nabla f(x_t)+\sqrt{2\epsilon}\xi_0).
\]
Note that the expression in the tangent space is exactly the same as unadjusted Langevin algorithm in Euclidean space.

\section{Missing proofs of Section \ref{MainResults}}
\subsection{Proof of Theorem \ref{convergence:continuous}}
In this section, we proof that the KL divergence decreases along the process evolving following Riemannian Langevin equation.
\\
\\
Firstly, need show that according to the SDE on manifold in local chart, the density function evolves according to Fokker-Planck/diffusion equation on this manifold.
\begin{customlemma}{\ref{lemma:continuous1}}
Suppose $\rho_t$ evolves following the Fokker-Planck equation (\ref{FP00}), then
\[
\frac{d}{dt}H(\rho_t|\nu)=-\int_M\rho_t(x)\norm{\grad\log\frac{\rho_t(x)}{\nu(x)}}^2dx
\]
where $dx$ is the Riemannian volume element.
\end{customlemma}

\begin{proof}
Since 
\begin{align}
\int_M\rho_t\frac{\partial}{\partial t}\log\frac{\rho_t}{\nu}dx&=\int_M\rho_t\frac{\partial}{\partial t}(\log\rho_t+f)dx
\\
&=\int_M\frac{\partial \rho_t}{\partial t}dx
\\
&=\frac{d}{dt}\int_M\rho_tdx=0,
\end{align}
we have
\begin{align}
\frac{d}{dt}H(\rho_t|\nu)&=\frac{d}{dt}\int_M\rho_t\log\frac{\rho_t}{\nu}dx
\\
&=\int_M\frac{d}{dt}\left(\rho_t\log\frac{\rho_t}{\nu}\right)dx
\\
&=\int_M\frac{\partial\rho_t}{\partial t}\log\frac{\rho_t}{\nu}dx+\int_M\rho_t\frac{\partial}{\partial t}\log\frac{\rho_t}{\nu}dx
\\
&=\int_M\frac{\partial\rho_t}{\partial t}\log\frac{\rho_t}{\nu}dx.
\end{align}
Plug in with diffusion equation 
\[
\frac{\partial\rho_t}{\partial t}=\Div(\rho_t\grad f+\grad \rho_t)
\]
and apply integration by parts, we obtain
\begin{align}
\frac{d}{dt}H(\rho_t|\nu)&=\int_M\Div(\rho_t\grad f+\grad \rho_t)\log\frac{\rho_t}{\nu}dx
\\
&=-\int_{M}\rho_t\norm{\grad\log\frac{\rho_t}{\nu}}^2dx
\\
&+\int_{\partial M}\log\frac{\rho_t}{\nu}\langle\rho_t\grad f+\grad\rho_t,n\rangle dx
\end{align}
Since $M$ is compact and has no boundary, the boundary integral equals to zero, then we have
\[
\frac{d}{dt}H(\rho_t|\nu)=-\int_M\rho_t\norm{\grad\log\frac{\rho_t}{\nu}}^2dx
\]
\end{proof}

\begin{customthm}{\ref{convergence:continuous}}
Suppose $\nu$ satisfies LSI with constant $\alpha>0$. Then along the Riemannian Langevin equation, i.e. the SDE (\ref{FP:local}) in local coordinate systems, the density $\rho_t$ satisfies
\[
H(\rho_t|\nu)\le e^{-2\alpha t}H(\rho_0|\nu).
\]
\end{customthm}

\begin{proof}
By LSI, we have
\[
\frac{d}{dt}H(\rho_t|\nu)\le-2\alpha H(\rho_t|\nu),
\]
multiplying both sides by $e^{\alpha t}$,
\[
e^{2\alpha t}\frac{d}{dt}H(\rho_t|\nu)\le-2\alpha e^{2\alpha t}H(\rho_t|\nu)
\]
and then
\[
e^{2\alpha t}\frac{d}{dt}H(\rho_t|\nu)+2\alpha e^{2\alpha t}H(\rho_t|\nu)=\frac{d}{dt}e^{2\alpha t}H(\rho_t|\nu)\le 0.
\]
Integrating for $0\le t\le s$, the result holds as
\[
e^{2\alpha s}H(\rho_s|\nu)-H(\rho_0|\nu)=\int_0^s\frac{d}{dt}e^{2\alpha t}H(\rho_t|\nu)\le 0.
\]
Rearranging and renaming $s$ by $t$, we conclude
\[
H(\rho_t|\nu)\le e^{-2\alpha t}H(\rho_t|\nu)
\]
\end{proof}

\subsection{Proof of Theorem \ref{convergence:discrete}}

\begin{lemma}\label{lemma:1}
Assume $\nu=e^{-f}$ is $L$-smooth. Then
\[
\mathbb{E}_{\nu}[\norm{\grad f}^2]\le nL.
\]
\end{lemma}

\begin{proof}
Since 
\[
\mathbb{E}_{\nu}[\norm{\grad f}^2]=\int_M\langle\grad f,\grad f\rangle e^{-f}dx=-\int_M\langle\grad e^{-f},\grad f\rangle dx,
\]
where $dx$ is the Riemannian volume element.
Integration by parts on manifold gives the following
\[
-\int_M\langle\grad e^{-f},\grad f\rangle dx=\int_Me^{-f}\Delta_Mfdx-\int_{\partial M}e^{-f}\langle\grad f,n\rangle ds
\]
where $ds$ is the area element on $\partial M$. By the assumption that $M$ is boundaryless, 
the integral on the boundary is 0. 
By the assumption $\Hess f$ is $L$-smooth and the fact that $\Hess f\ge\frac{1}{n}\Delta_Mf$, we conclude $\mathbb{E}_{\nu}[\norm{\grad f}^2]\le nL$
\end{proof}

\begin{lemma}\label{lemma:2}
Suppose $\nu$ satisfies Talagrand inequality with constant $\alpha>0$ and $L$-smooth. Then for any $\rho$,
\[
\mathbb{E}_{\rho}[\norm{\grad f}^2]\le\frac{4L^2}{\alpha}H(\rho|\nu)+2nL.
\]
\end{lemma}

\begin{proof}
Let $x\sim\rho$ and $y\sim\nu$ with optimal coupling $(x,y)$ so that
\[
\mathbb{E}[d(x,y)^2]=W_2(\rho,\nu)^2.
\]
$\grad f$ is $L$-Lipschitz from the assumption that $f$ is $L$-smooth. So we have the following inequality:
\begin{align}
\norm{\grad f(x)}&\le\norm{\grad f(x)-\Gamma_y^x\grad f(y)}+\norm{\Gamma_y^x\grad f(y)}
\\
&\le Ld(x,y)+\norm{\Gamma_y^x\grad f(y)}
\\
&=Ld(x,y)+\norm{\grad f(y)}
\end{align}
where the equality follows from that parallel transport is an isometry. The same arguments as V-W gives 
\[
\norm{\grad f(x)}^2\le(Ld(x,y)+\norm{\grad f(y)})^2\le 2Ld(x,y)^2+2\norm{\grad f(y)}^2
\]
and
\begin{align}
\mathbb{E}_{\rho}[\norm{\grad f(x)}^2]&\le 2L^2\mathbb{E}[d(x,y)^2]+2\mathbb{E}_{\nu}[\norm{\grad f(y)}^2]
\\
&=2L^2W_2(\rho,\nu)^2+2\mathbb{E}_{\nu}[\norm{\grad f(y)}^2].
\end{align}
By Talagrand inequality and previous lemma, the result follows.
\end{proof}

\begin{assumption}
We next assume the existence of constants shown in the convergence result. Let the joint distribution $p_{0t}(x_0,x)$ be differentiable and assume that $\frac{\frac{\partial^2p_{0t}(x_0,x)}{\partial x_i\partial x_j}\log\frac{p_t}{\nu}}{p_{0t}(x_0,x)}$ is bounded by $K_2$ and $\frac{\abs{\frac{\partial p_{0t}(x_0,x)}{\partial x_i}}\log\frac{p_t}{\nu}}{p_{0t}(x_0,x)}$ is bounded by $K_3$. 
\end{assumption}

\begin{lemma}\label{lemma:4}
Suppose $\nu$ satisfies LSI with constant $\alpha>0$ and is $L$-smooth. If $\epsilon$ small enough, then along each step,
\[
H(p_{\epsilon}|\nu)\le e^{-\alpha\epsilon}H(p_0|\nu)+4\epsilon^2(2nL^2+2n^3K_2C+nK_3K_4)
\]
for small $\epsilon$, and
\[
H(p_{k+1}|\nu)\le e^{-\alpha\epsilon}H(p_k|\nu)+4\epsilon^2(2nL^2+2n^3K_2C+nK_3K_4).
\]
for all $k\in\mathbb{N}$. 
\end{lemma}

\begin{proof}
According to \cite{ItohTanaka}, the exponential map $\Exp_x$ is a diffeomorphism on almost all the manifold, i.e. let $\bar{U}_x$ be the closed set of vectors in $T_xM$ for which $\gamma(t)=\Exp_x(tv), t\in[0,1]$ is length minimizing, and $U_x$ be the interior and $\partial \bar{U}_x$ be its boundary. Then the exponential map is a diffeomorphism on $U_x$ and $\Exp_x(\partial\bar{U}_x)$ has measure zero.

In normal coordinates, the discretized SDE has the form of 
\[
dx_t=-\grad f(x_0)dt+\sqrt{2g^{-1}(x_0)}dB_t
\]
and the Fokker-Planck equation of this SDE is 

\begin{align}
\frac{\partial p_{t|0}(x_t|x_0)}{\partial t} &=\sum_i\frac{\partial}{\partial x_i}((\sum_jg^{ij}\frac{\partial f}{\partial x_j}(x_0))p_{t|0}(x_t|x_0))+\sum_{i,j}\frac{\partial^2}{\partial x_i\partial x_j}g^{-1}(x_0)p_{t|0}(x_t|x_0)
\\
&=\Div(p_{t|0}(x_t|x_0)\grad f(x_0))+\sum_{i,j}g^{-1}\frac{\partial^2}{\partial x_i\partial x_j}p_{t|0}(x_t|x_0)+\sum_ib_i\frac{\partial}{\partial x_i}p_{t|0}(x_t|x_0)
\\
&+\sum_{i,j}\frac{\partial^2}{\partial x_i\partial x_j}g^{-1}(x_0)p_{t|0}(x_t|x_0)-(\sum_{i,j}g^{-1}\frac{\partial^2}{\partial x_i\partial x_j}p_{t|0}(x_t|x_0)+\sum_ib_i\frac{\partial}{\partial x_i}p_{t|0}(x_t|x_0))
\\
&=\Div(p_{t|0}(x_t|x_0)\grad f(x_0))+\Delta_Mp_{t|0}(x_t|x_0)
\\
&+\sum_{i,j}(g^{ij}(x_0)-g^{ij}(x_t))\frac{\partial^2}{\partial x_i\partial x_j}p_{t|0}(x_t|x_0)-\sum_ib_i\frac{\partial}{\partial x_i}p_{t|0}(x_t|x_0)
\end{align}

\begin{align}
\frac{\partial p_t(x)}{\partial t}&=\int_{\mathbb{R}^n}\frac{\partial p_{t|0}(x|x_0)}{\partial t}p_0(x_0)\sqrt{\abs{g}}dx_0
\\
&=\int_{\mathbb{R}^n}\left(\sum_i\frac{\partial}{\partial x_i}((\sum_jg^{ij}\frac{\partial f}{\partial x_j}(x_0))p_{t|0}(x_t|x_0))+\sum_{i,j}\frac{\partial^2}{\partial x_i\partial x_j}g^{-1}(x_0)p_{t|0}(x_t|x_0)\right)p_0(x_0)\sqrt{\abs{g}}dx_0
\\
&=\int_{\mathbb{R}^n}\left(\Div(p_{t|0}(x|x_0)\grad f(x_0))+\Delta_Mp_{t|0}(x|x_0)\right)p_0(x_0)\sqrt{\abs{g}}dx_0
\\
&+\int_{\mathbb{R}^n}\left(\sum_{i,j}(g^{ij}(x_0)-g^{ij}(x))\frac{\partial^2}{\partial x_i\partial x_j}p_{t|0}(x|x_0)\right)p_0(x_0)\sqrt{\abs{g}}dx_0
\\
&-\int_{\mathbb{R}^n}\left(\sum_ib_i\frac{\partial}{\partial x_i}p_{t|0}(x|x_0)\right)p_0(x_0)\sqrt{\abs{g}}dx_0
\end{align}

\begin{align}
&\int_{\mathbb{R}^n}\left(\Div(p_{t|0}(x|x_0)\grad f(x_0))+\Delta_Mp_{t|0}(x|x_0)\right)p_0(x_0)\sqrt{\abs{g}}dx_0
\\
&=\int_{\mathbb{R}^n}\Div(p_{t0}(x,x_0)\grad f(x_0))\sqrt{\abs{g}}dx_0+\Delta_Mp_t(x)
\\
&=\Div\left(p_t(x)\int_{\mathbb{R}^n}p_{0|t}(x_0|x)\grad f(x_0)dx_0\right)+\Delta_Mp_t(x)
\\
&=\Div\left(p_t(x)\mathbb{E}_{p_{0|t}}[\grad f(x_0)|x_t=x]\right)+\Delta_Mp_t(x)
\end{align}

\begin{align}
&\int_{\mathbb{R}^n}\left(\sum_{i,j}(g^{ij}(x_0)-g^{ij}(x))\frac{\partial^2}{\partial x_i\partial x_j}p_{t|0}(x|x_0)\right)p_0(x_0)\sqrt{\abs{g}}dx_0
\\
&=\sum_{i,j}\int_{\mathbb{R}^n}(g^{ij}(x_0)-g^{ij}(x))\frac{\partial^2}{\partial x_i\partial x_j}p_{t|0}(x|x_0)p_0(x_0)\sqrt{\abs{g}}dx_0
\\
&\le \sum_{i,j}\int_{\mathbb{R}^n}\abs{g^{ij}(x_0)-g^{ij}(x)}\cdot\abs{\frac{\partial^2}{\partial x_i\partial x_j}p_{t|0}(x|x_0)p_0(x_0)}\sqrt{\abs{g}}dx_0
\\
&\le \sum_{i,j}\int_{\mathbb{R}^n}O(\norm{x-x_0}^2)\abs{\frac{\partial^2}{\partial x_i\partial x_j}p_{t|0}(x|x_0)p_0(x_0)}\sqrt{\abs{g}}dx_0
\\
&=\sum_{i,j}\int_{\mathbb{R}^n}O(\norm{x-x_0}^2)\abs{\frac{\partial^2}{\partial x_i\partial x_j}p_{0t}(x_0,x)}\sqrt{\abs{g}}dx_0
\\
&\le n^2K_1\int_{\mathbb{R}^n}O(\norm{x-x_0}^2)\sqrt{\abs{g}}dx_0
\end{align}
where $K_1$ is the upper bound of $\abs{\frac{\partial^2}{\partial x_i\partial x_j}p_{0t}(x_0,x)}$

\begin{align}\label{eq:n1}
&\int_{\mathbb{R	}^n}\left(\int_{\mathbb{R}^n}\left(\sum_{i,j}(g^{ij}(x_0)-g^{ij}(x))\frac{\partial^2}{\partial x_i\partial x_j}p_{t|0}(x|x_0)\right)p_0(x_0)\sqrt{\abs{g}}dx_0\right)\log\frac{p_t}{\nu}\sqrt{\abs{g}}dx
\\
&\le\sum_{ij}\int_{\mathbb{R}^n}\left(\int_{\mathbb{R}^n}\abs{g^{ij}(x_0)-g^{ij}(x)}\abs{\frac{\partial^2}{\partial x_i\partial x_j}p_{0t}(x_0,x)}\sqrt{\abs{g}}dx_0\right)\abs{\log\frac{p_t}{\nu}}\sqrt{\abs{g}}dx
\end{align}

Let
\[
\tilde{p}(x_0,x)=\frac{\frac{\partial^2p_{t0}(x_0,x)}{\partial x_i\partial x_j}\log\frac{p_t}{\nu}}{p_{t0}(x_0,x)}
\]
and assume that $\abs{\tilde{p}(x_0,x)}$ is bounded by $K_2$. Then (\ref{eq:n1}) is bounded by $n^2K_2\mathbb{E}_{p_{t0}}\left[O\left(\norm{-t\grad f(x_0)+\sqrt{2t}z}^2\right)\right]$.

\begin{align}
&\int_{\mathbb{R}^n}\left(\int_{\mathbb{R}^n}\left(\sum_{ij}b_i\frac{\partial p_{0t}(x_0,x)}{\partial x_i}\right)\sqrt{\abs{g(x_0)}}dx_0\right)\log\frac{p_t}{\nu}\sqrt{\abs{g(x)}}dx
\\
&=\sum_i\int_{\mathbb{R}^n\times\mathbb{R}^n}b_i(x)\frac{\partial p_{0t}(x_0,x)}{\partial x_i}\log\frac{p_t}{\nu}d(x_0\times x)
\\
&\le K_3\int_{\mathbb{R}^n\times\mathbb{R}^n}\abs{b_i(x)}p_{0t}(x_0,x)d(x_0\times x)\ \ (\text{suppose}\ \ K_3\ge\frac{\abs{\frac{\partial p_{0t}(x_0,x)}{\partial x_i}}\log\frac{p_t}{\nu}}{p_{0t}(x_0,x)})
\\
&=K_3\sum_i\mathbb{E}_{p_{0t}}[\abs{b_i(x)}]
\\
&=K_3\sum_i\mathbb{E}_{p_{0t}}\left[\abs{b_i(x_0-t\grad f(x_0)+\sqrt{2t}z)}\right]
\end{align}

\begin{align}
&b_i(x_0-t\grad f(x_0)+\sqrt{2t}z)
\\
&=b_i(x_0)-t\langle\nabla b_i(x_0),\nabla f(x_0)\rangle+\sqrt{2t}\langle\nabla b_i(x_0),z\rangle+t\langle z,\nabla^2b_i(x_0)z\rangle.
\end{align}
and then 
\begin{align}
\mathbb{E}_{p_{0t}}\left[\abs{b_i(x_0-t\grad f(x_0)+\sqrt{2t}z)}\right]\le tK_4
\end{align}
where $K_4$ is determined by the expectation of $\langle \nabla b_i(x_0),\nabla f(x_0)\rangle$ and $\langle z,\nabla^2b_i(x_0)z\rangle$. So we have
\[
\int_{\mathbb{R}^n}\left(\int_{\mathbb{R}^n}\left(\sum_{ij}b_i\frac{\partial p_{0t}(x_0,x)}{\partial x_i}\right)\sqrt{\abs{g(x_0)}}dx_0\right)\log\frac{p_t}{\nu}\sqrt{\abs{g(x)}}dx\le tnK_3K_4.
\]

\begin{align}
\frac{d}{dt}H(p_t|\nu)=\int_{\mathbb{R}^n}\left(\int_{\mathbb{R}^n}\left(\Div(p_{t|0}(x|x_0)\grad f(x_0))+\Delta_Mp_{t|0}(x|x_0)\right)p_0(x_0)\sqrt{\abs{g}}dx_0\right)\log\frac{p_t}{\nu}dx
\end{align}

\begin{align}
\frac{d}{dt}H(p_t|\nu)&\le-\frac{3}{4}J+\frac{4t^2L^4}{\alpha}H(p_0|\nu)+2t^2nL^3+2tnL^2
\\
&+n^2K_2\mathbb{E}_{p_{t0}}\left[O\left(\norm{-t\grad f(x_0)+\sqrt{2t}z}^2\right)\right]+tnK_3K_4
\\
&\le-\frac{3}{4}J+\frac{4t^2L^4}{\alpha}H(p_0|\nu)+2t^2nL^3+2tnL^2
\\
&+n^2K_2C\left(\frac{4t^2L^2}{\alpha}H(p_0|\nu)+2t^2nL+2tn\right)+tnK_3K_4
\\
&=-\frac{3}{4}J+\frac{4t^2L^4+4t^2L^2n^2K_2C}{\alpha}H(p_0|\nu)+2t^2(nL^3+n^3K_2CL)+t(2nL^2+2n^3K_2C+nK_3K_4)
\end{align}

Let $t\le\epsilon\le \frac{2nL^2+2n^3K_2C+nK_3K_4}{2(nL^3+n^3K_2CL)}$, we have 
\begin{align}
\frac{d}{dt}H(p_t|\nu)\le -\frac{3\alpha}{2}H(p_t|\nu)+\frac{4\epsilon^2(L^4+L^2n^2k_2C)}{\alpha}H(p_0|\nu)+2\epsilon(2nL^2+2n^3K_2C+nK_3K_4)
\end{align}
Multiplying both sides by $e^{\frac{3\alpha}{2}t}$, we have
\[
\frac{d}{dt}\left(e^{\frac{3\alpha}{2}t}H(p_t|\nu)\right)\le e^{\frac{3\alpha}{2}t}\left(\frac{4\epsilon^2(L^4+L^2n^2k_2C)}{\alpha}H(p_0|\nu)+2\epsilon(2nL^2+2n^3K_2C+nK_3K_4)\right)
\]
and integrating for $t\in[0,\epsilon]$,

\begin{align}
e^{\frac{3}{2}\alpha\epsilon}H(p_{\epsilon}|\nu)-H(p_0|\nu)&\le\frac{2(e^{\frac{3\alpha}{2}\epsilon}-1)}{3\alpha}\left(\frac{4\epsilon^2(L^4+L^2n^2k_2C)}{\alpha}H(p_0|\nu)+2\epsilon(2nL^2+2n^3K_2C+nK_3K_4)\right)
\\
&\le2\epsilon\left(\frac{4\epsilon^2(L^4+L^2n^2k_2C)}{\alpha}H(p_0|\nu)+2\epsilon(2nL^2+2n^3K_2C+nK_3K_4)\right)
\end{align}

So
\begin{align}
H(p_{\epsilon}|\nu)&\le e^{-\frac{3}{2}\alpha\epsilon}\left(\frac{8\epsilon^3(L^4+L^2n^2K_2C)}{\alpha}+1\right)H(p_0|\nu)+e^{-\frac{3}{2}\alpha\epsilon}4\epsilon^2(2nL^2+2n^3K_2C+nK_3K_4)
\\
&\le e^{-\frac{3}{2}\alpha\epsilon}\left(\frac{8\epsilon^3(L^4+L^2n^2K_2C)}{\alpha}+1\right)H(p_0|\nu)+4\epsilon^2(2nL^2+2n^3K_2C+nK_3K_4).
\end{align}

If $1+\frac{8\epsilon^3(L^4+L^2n^2K_2C)}{\alpha}\le 1+\frac{\alpha\epsilon}{2}\le e^{\frac{1}{2}\alpha\epsilon}$, or $\epsilon\le\frac{\alpha}{4L\sqrt{L^2+n^2K_2C}}$,
\[
H(p_{\epsilon}|\nu)\le e^{-\alpha\epsilon}H(p_0|\nu)+4\epsilon^2(2nL^2+2n^3K_2C+nK_3K_4),
\]
and then
\[
H(p_{k+1}|\nu)\le e^{-\alpha\epsilon}H(p_k|\nu)+4\epsilon^2(2nL^2+2n^3K_2C+nK_3K_4).
\]

\end{proof}

\begin{customthm}{\ref{convergence:discrete}}
Suppose $M$ is a compact manifold without boundary and $R$ is the Riemann curvature, $\nu=e^{-f}$ a density on $M$ with $\alpha>0$ the log-Sobolev constant. Then there exists a global constant $K_2,K_3,K_4,C$, such that for any $x_0\sim\rho_0$ with $H(\rho_0|\nu)\le\infty$, the iterates $x_k\sim\rho_k$ of GLA with stepsize $\epsilon\le\min\{\frac{\alpha}{4L\sqrt{L^2+n^2K_2C}},\frac{2nL^2+2n^3K_2C+nK_3K_4}{2(nL^3+n^3K_2CL)},\frac{1}{2L},\frac{1}{2\alpha}\}$ satisfty

\[
H(p_k|\nu)\le e^{-\alpha k\epsilon}H(p_0|\nu)+\frac{16\epsilon}{3\alpha}(2nL^2+2n^3K_2C+nK_3K_4)
\]
\end{customthm}

\begin{proof}
\begin{align}
H(p_k|\nu)&\le e^{-\alpha k\epsilon}H(p_0|\nu)+\frac{1-e^{-\alpha k\epsilon}}{1-e^{-\alpha\epsilon}}4\epsilon(2nL^2+2n^3K_2C+nK_3K_4)
\\
&\le e^{-\alpha k\epsilon}H(p_0|\nu)+\frac{4}{3\alpha\epsilon}4\epsilon^2(2nL^2+2n^3K_2C+nK_3K_4)
\\
&= e^{-\alpha k\epsilon}H(p_0|\nu)+\frac{16\epsilon}{3\alpha}(2nL^2+2n^3K_2C+nK_3K_4)
\end{align}

\end{proof}

\newpage
\section{Experiments}
As mentioned before, for simplicity, we can implement GLA without using the exponential map where a geodesic ODE solver is required, especially for the case when $M$ is a submanifold of $\mathbb{R}^n$. In general, the retraction map from $T_xM$ to $M$ is used in optimization on Riemannian manifold \cite{Boumal1}, as a replacement of exponential map. In this section, we give experiments on sampling from distributions on the unit sphere in comparison of exponential map and orthogonal projection as a retraction in the geodesic step of GLA. 

The experiments are designed to verify the following properties:
\begin{enumerate}
\item GLA captures the target distribution $e^{-f}$ as expected;
\item The projection map behaves well in replacing the exponential map without solving geodesic equations.
\end{enumerate}
In each set of figures, (a) is the landscape of the ideal distribution, (b) and (c) are the results with small number of iterations for exponential map and projection, (d) and (e) are enhanced with large number of iterations. We start with the definition of the general retraction in optimization on manifold.
\begin{definition}[Retraction]
A retraction on a manifold $M$ is a smooth mapping $\Retr$ from the tangent bundle $TM$ to $M$ satisfying properties 1 and 2 below: Let $\Retr_x:T_xM\rightarrow M$ denote the restriction of $Retr$ to $T_xM$.
\begin{enumerate}
\item $\Retr_x(0)=x$, where $0$ is the zero vector in $T_xM$.
\item The differential of $\Retr_x$ at $0$ is the identity map.
\end{enumerate}
\end{definition}

Suppose $M$ is a submanifold of $\mathbb{R}^n$ with positive codimension. Denote $\Proj_{T_xM}$ the orthogonal projection to the tangent space at $x$, then the retraction can be defined as $\Retr_x(v)=\Proj_M(x+v)$. The GLA on a submanifold of $\mathbb{R}^n$ can be written as
\begin{equation}\label{VGLA}
x_{k+1}=\Retr_x\left(\Proj_{T_xM}(-\epsilon\nabla f(x_k)+\sqrt{2\epsilon}\xi_0)\right)
\end{equation}
If $M=S^{n-1}$ be the unit sphere in $\mathbb{R}^n$, then $\Retr_x(v)=\frac{x+v}{\norm{x+v}}$.

\begin{figure}
\centering
\subfigure[Ideal distribution of $e^{-f}$]{\includegraphics[width=0.4\textwidth]{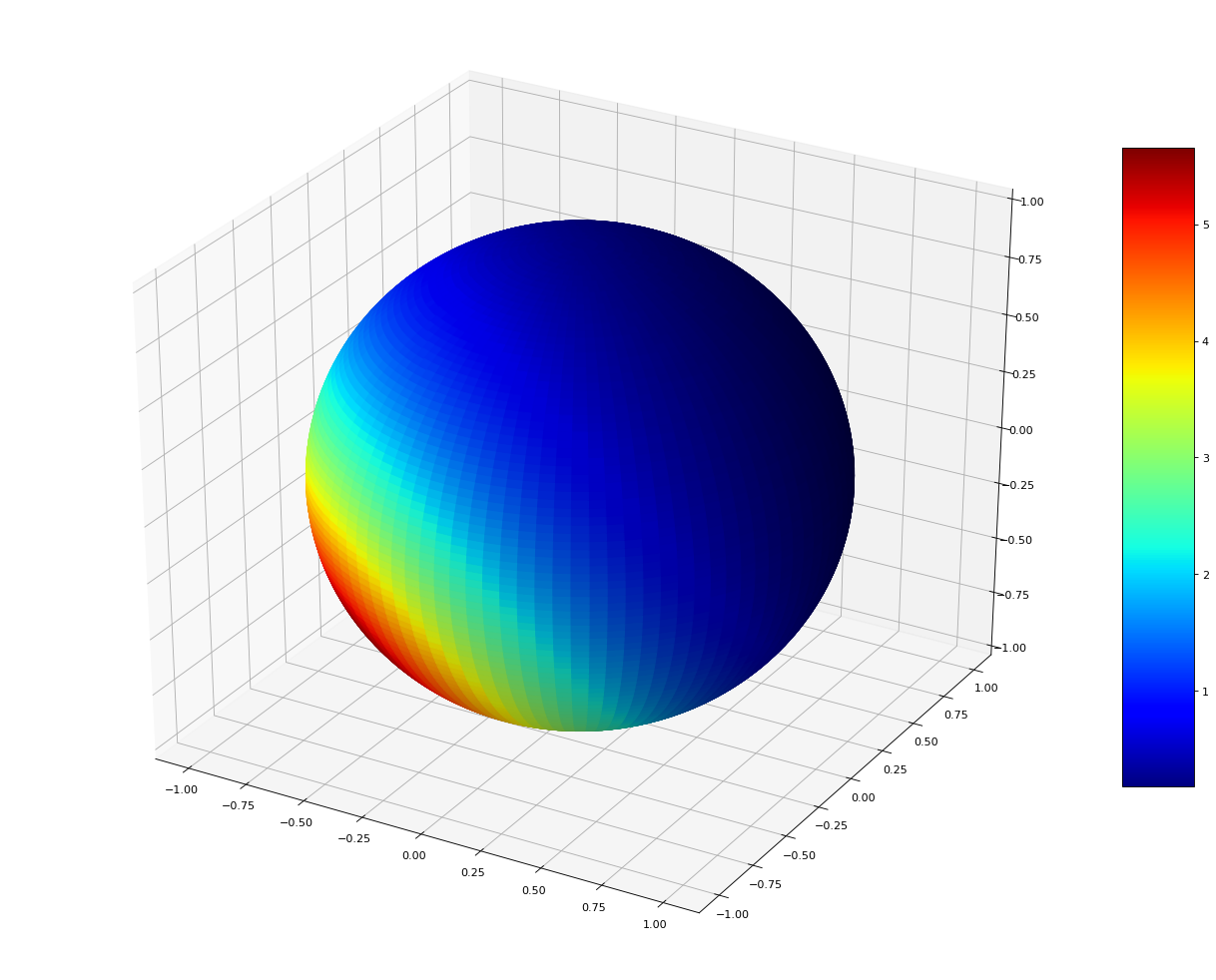}}
\\
\subfigure[$\Exp_x(v)$, iteration: 10k]{\label{fig:a}\includegraphics[width=0.4\textwidth]{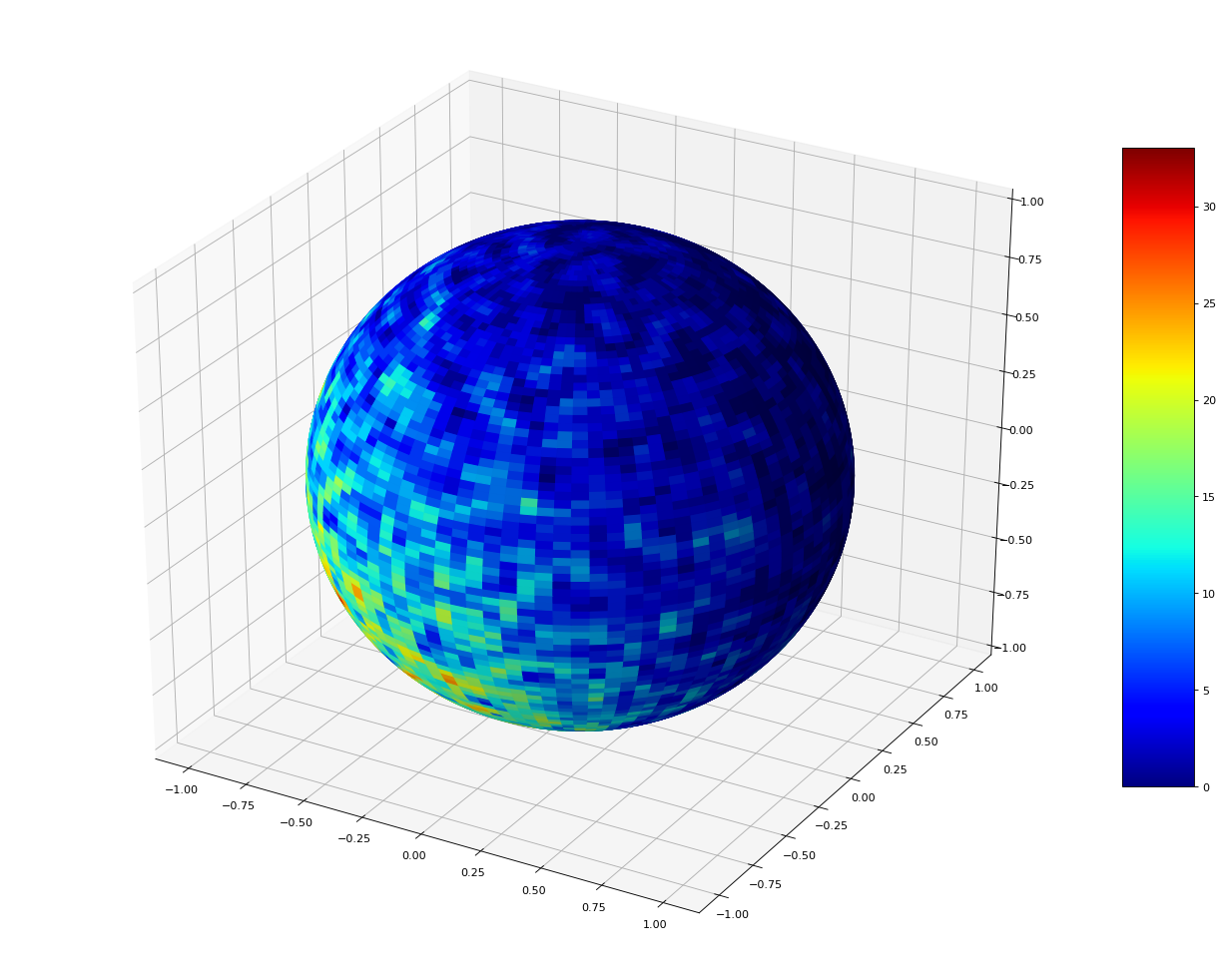}}
\hspace{0.1\textwidth}
\subfigure[$\Retr_x(v)$, iteration: 10k]{\label{fig:b}\includegraphics[width=0.4\textwidth]{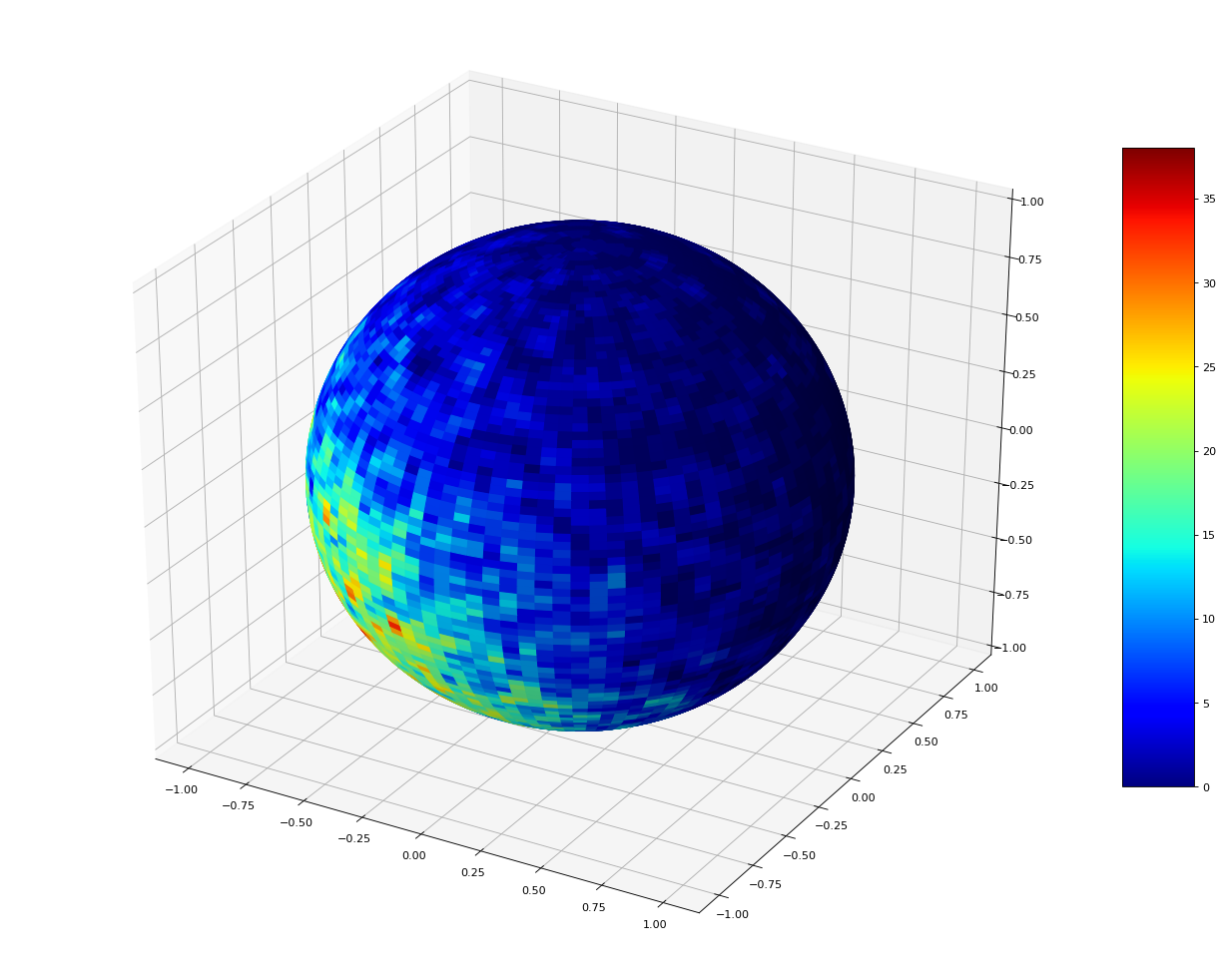}}
\subfigure[$\Exp_x(v)$, iteration: 100k]{\includegraphics[width=0.4\textwidth]{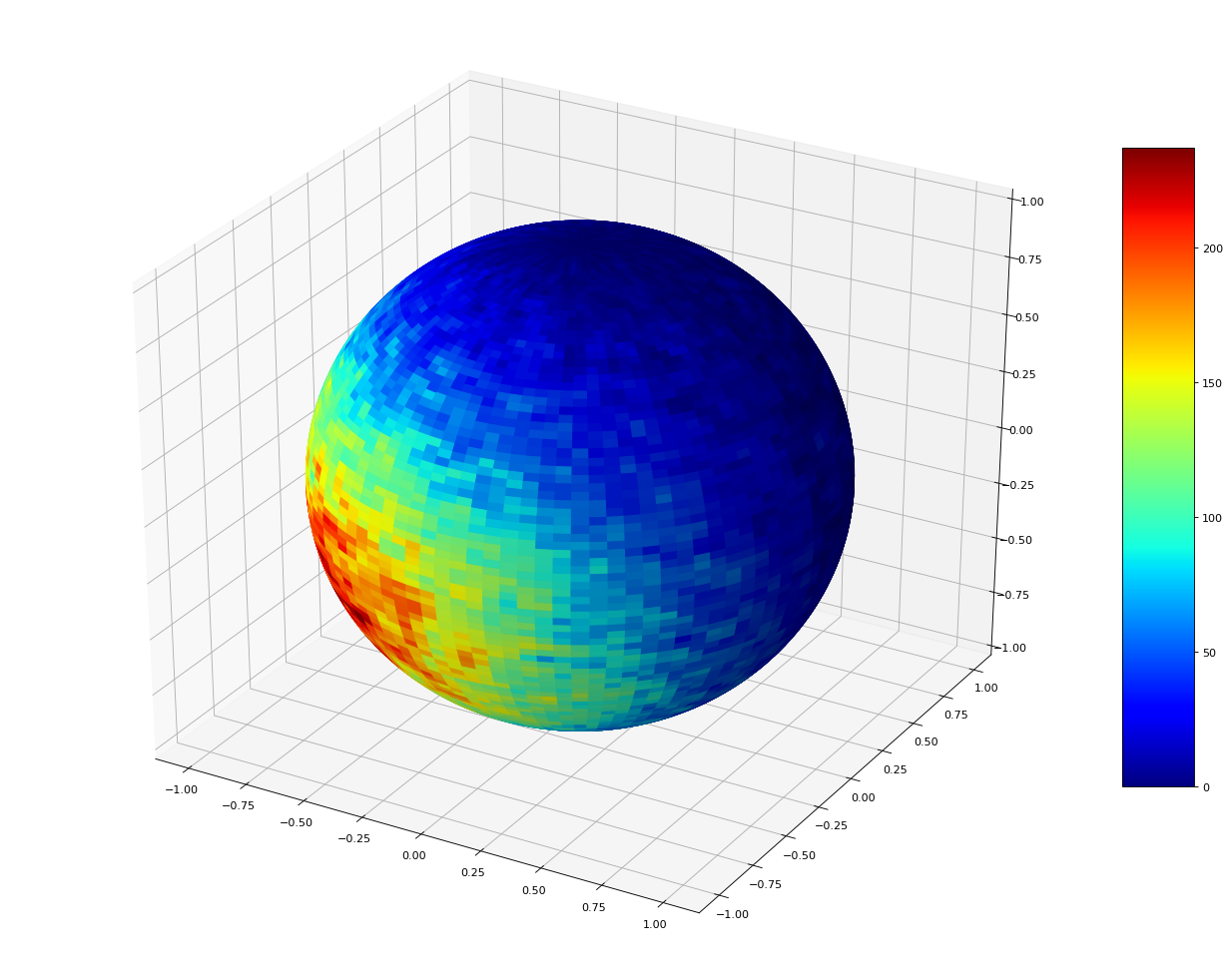}}
\hspace{0.1\textwidth}
\subfigure[$\Retr_x(v)$, iteration: 100k]{\includegraphics[width=0.4\textwidth]{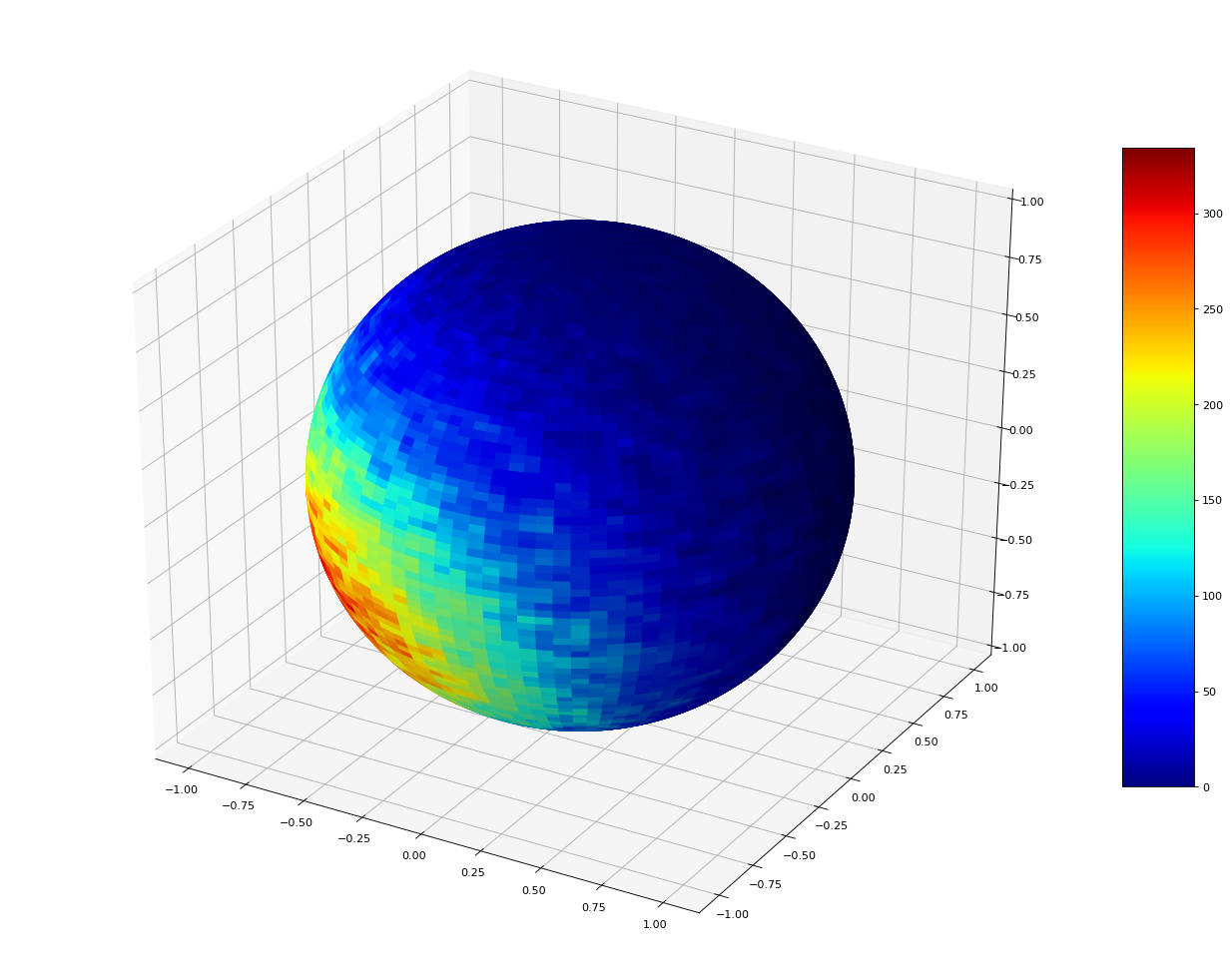}}
\caption{$f(x)=x_1+x_2+x_3$, stepsize $\epsilon=0.1$.}
\label{Fig2}
\end{figure}

\begin{figure}
\centering
\subfigure[Ideal distribution of $e^{-f}$]{\includegraphics[width=0.4\textwidth]{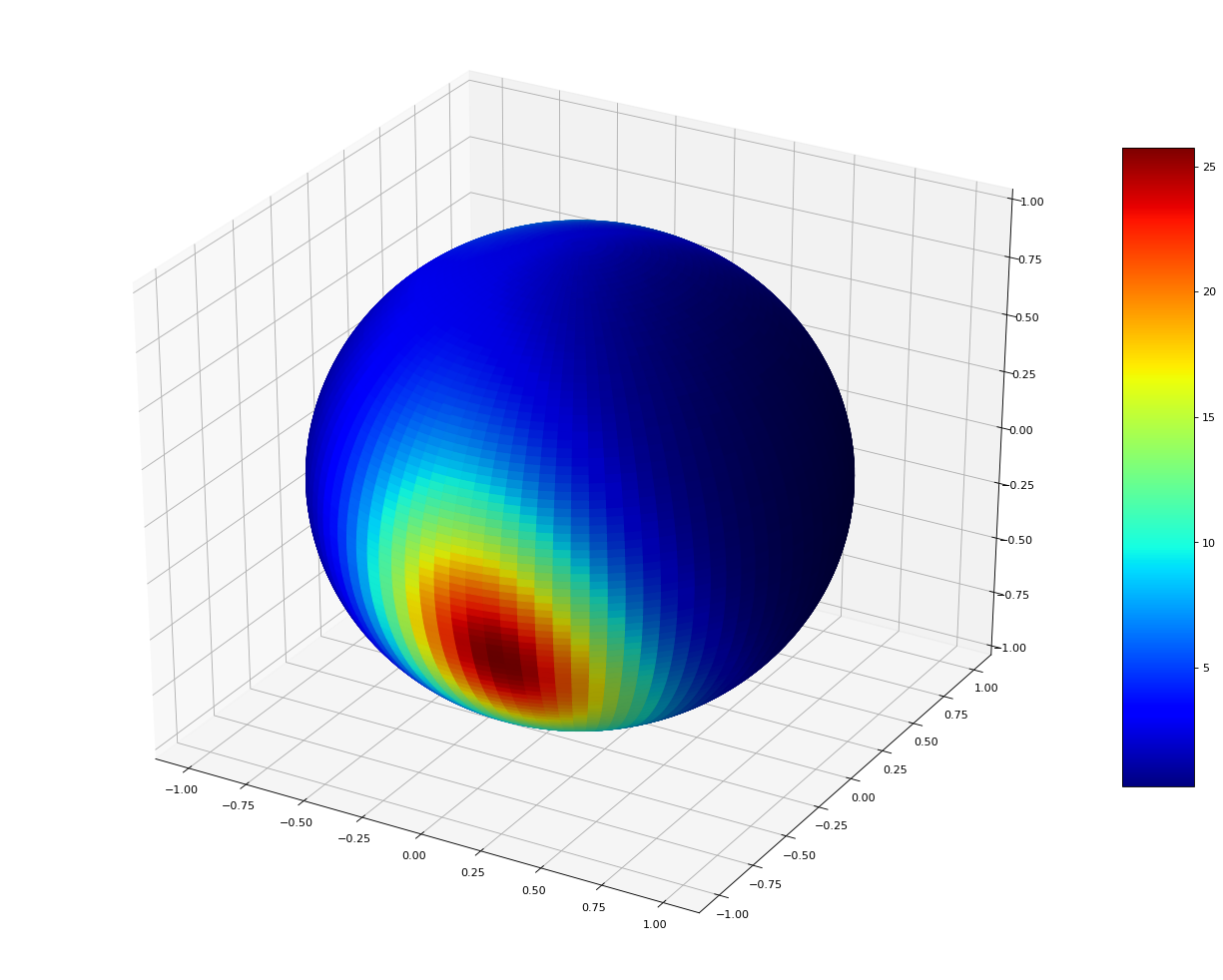}}
\\
\subfigure[$\Exp_x(v)$, iteration: 10k]{\label{fig:ax}\includegraphics[width=0.4\textwidth]{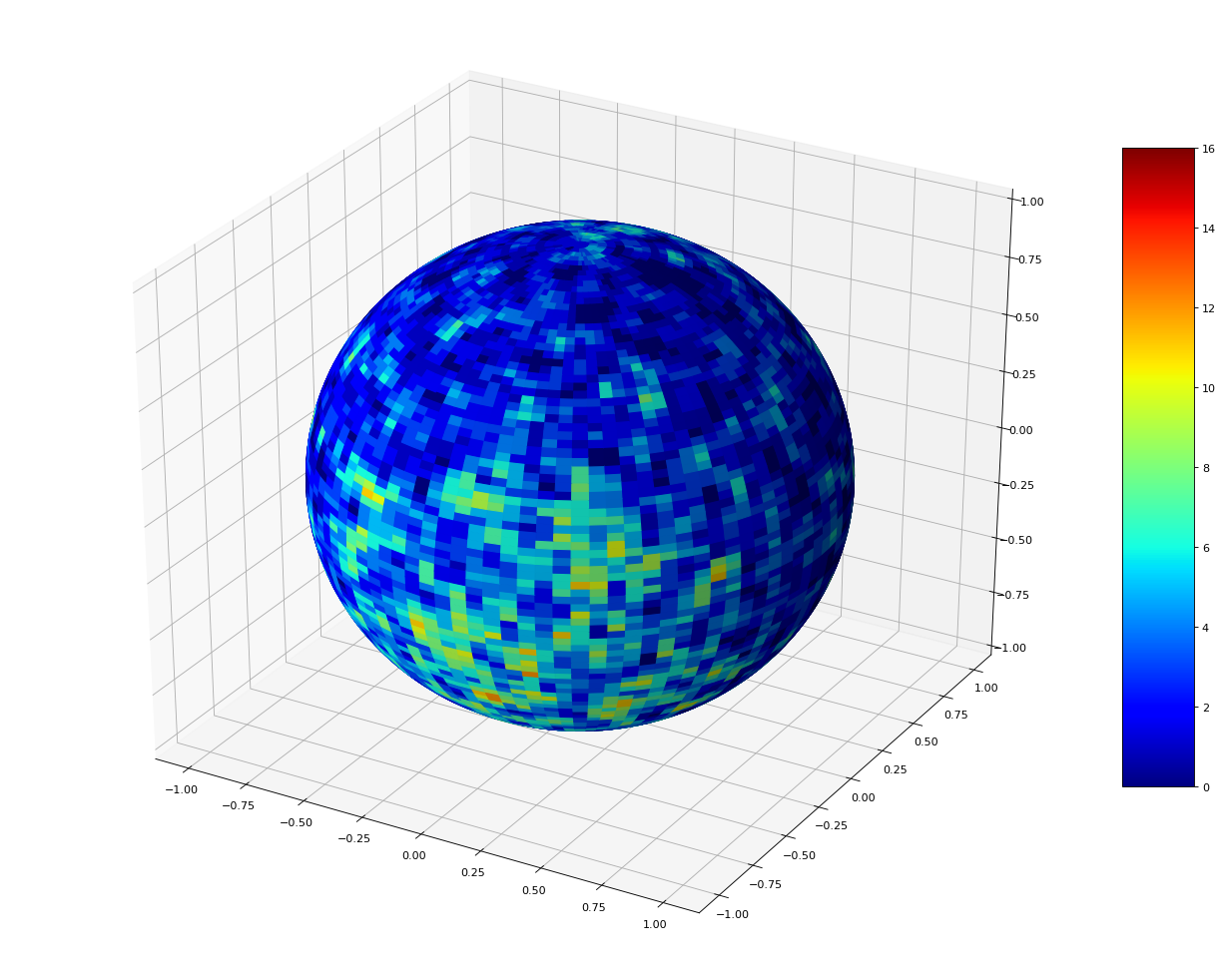}}
\hspace{0.1\textwidth}
\subfigure[$\Retr_x(v)$, iteration: 10k]{\label{fig:bx}\includegraphics[width=0.4\textwidth]{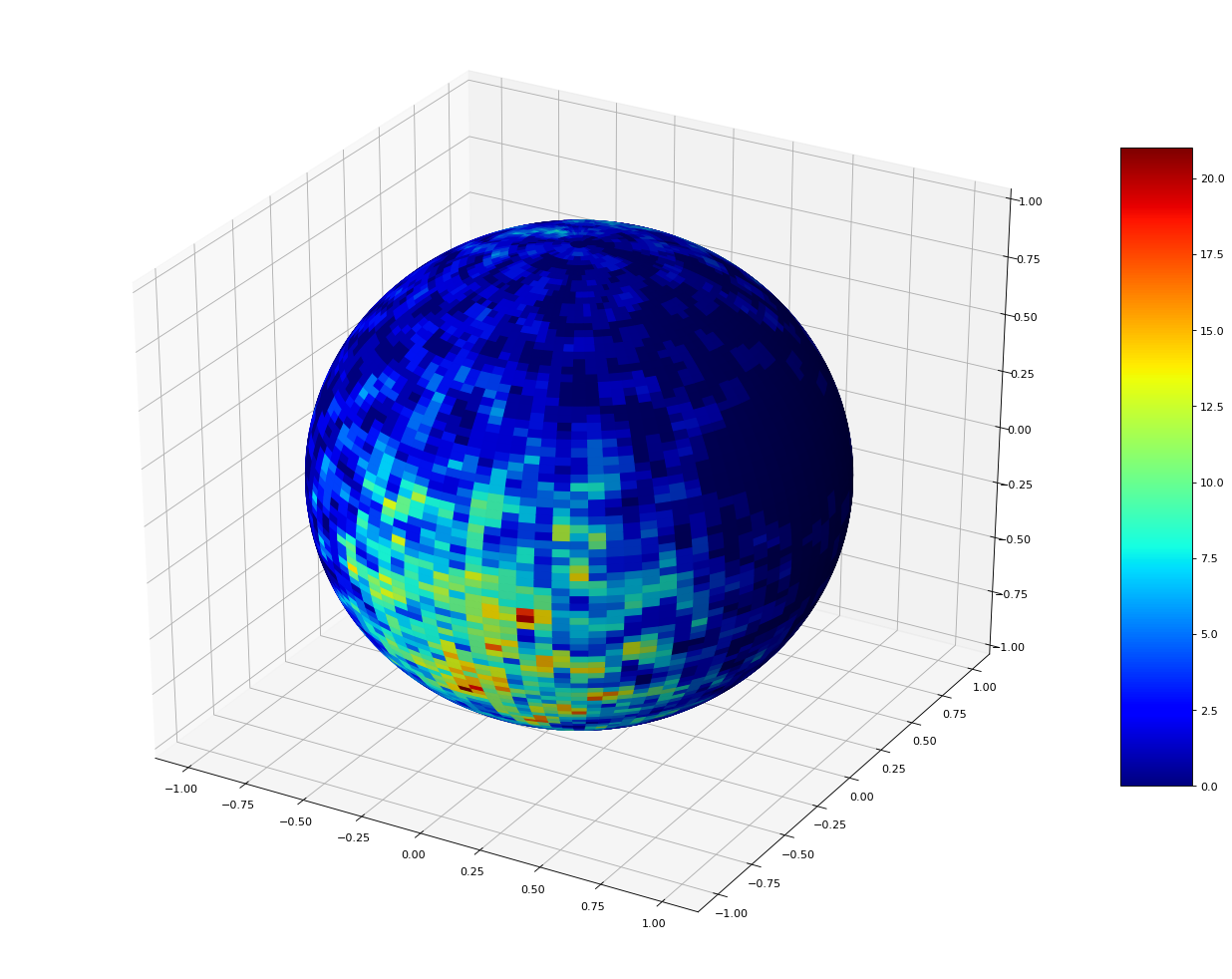}}
\subfigure[$\Exp_x(v)$, iteration: 100k]{\includegraphics[width=0.4\textwidth]{exp_poly3d_100ksamples_0dot1stepsize}}
\hspace{0.1\textwidth}
\subfigure[$\Retr_x(v)$, iteration: 100k]{\includegraphics[width=0.4\textwidth]{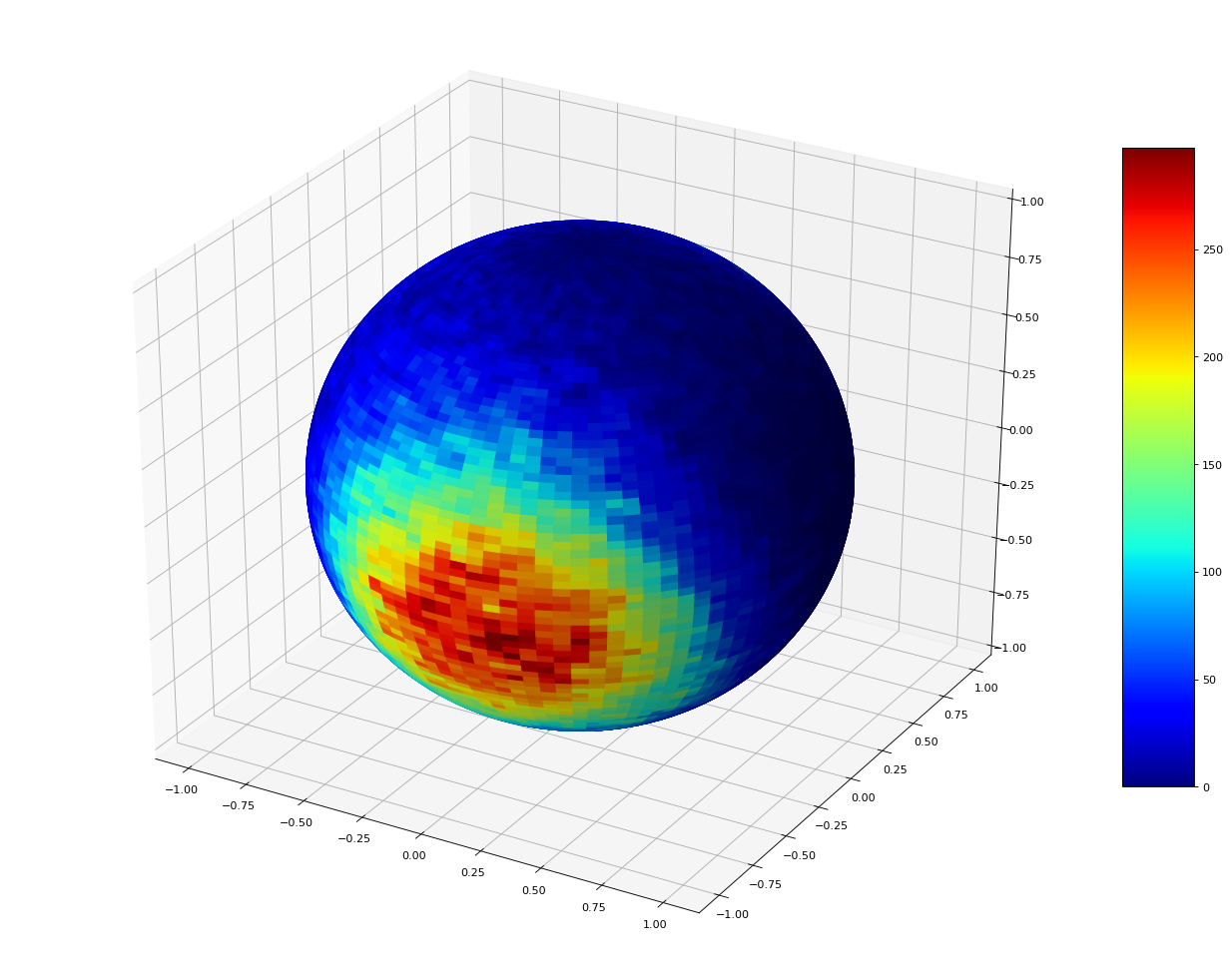}}
\caption{$f(x)=x_1^2+3.05x_2^2-0.9x_3^2+1.1x_1x_2+-1.02x_2x_3+2.1x_3x_1$, setpsize $\epsilon=0.1$.}
\label{Fig3}
\end{figure}

\end{document}